%% file: main.tex
\documentclass[10pt,twocolumn,letterpaper]{article}

\usepackage[final]{cvpr}      % To produce the REVIEW version
\input{preamble}
\definecolor{cvprblue}{rgb}{0.21,0.49,0.74}
\usepackage{algorithm}
\usepackage{algorithmicx}
\usepackage{algpseudocode}
\usepackage{tikz}
\usepackage{amsmath}
\usepackage{makecell}
\usepackage{tikz}
\usepackage{amsthm}
\usepackage{amssymb}
\usepackage{booktabs}
\usepackage{rotating}
\usepackage{tabularx}
\usepackage{rotating}
\usepackage{multirow}
\usepackage{hyperref}
\usepackage{svg}
\usetikzlibrary{shapes, arrows, positioning, fit, calc}
\def\paperID{*****} % *** Enter the Paper ID here
\def\confName{CVPR}
\def\confYear{2026}

\title{Hilbert-Geo: Solving Solid Geometric Problems by Neural-Symbolic Reasoning}

\author{
Ruoran Xu$^1$, Haoyu Cheng$^1$, Bin Dong$^2$, Qiufeng Wang$^{1,\dagger}$ \\
$^1$Xi'an Jiaotong-Liverpool University \quad $^2$Ricoh Software Research Center Beijing Co.,Ltd.
}

\begin{document}
\maketitle
\begingroup
\renewcommand{\thefootnote}{\fnsymbol{footnote}}
\footnotetext[2]{Corresponding author: qiufeng.wang@xjtlu.edu.cn}
\endgroup
\input{sec/0_abstract}    
\input{sec/1_intro}
\input{sec/2_formatting}
\input{sec/3_finalcopy}
\input{sec/4}
\input{sec/5}
\input{sec/6}
\input{sec/7}
\input{sec/9}

\section*{Acknowledgments}
This work was supported by National Natural Science Foundation of China under No.62436009, and Jiangsu Science and Technology Programme BK20251812, and Open Research Fund of The State Key Laboratory of Multimodal Artificial Intelligence Systems.

{
    \small
    \bibliographystyle{ieeenat_fullname}
    \bibliography{main}
}
\input{sec/aA}
\input{sec/aB}
\input{sec/aC}
\input{sec/aD}
\input{sec/aE}

\end{document}

%% file: sec/0_abstract.tex
\begin{abstract}
Geometric problem solving, as a typical multimodal reasoning problem, has attracted much attention and made great progress recently, however most of works focus on plane geometry while usually fail in solid geometry due to 3D spatial diagrams and complex reasoning. To bridge this gap, we introduce Hilbert-Geo, the first unified formal language framework for solid geometry, including an extensive predicate library and a dedicated theorem bank. 
Based on this framework, we propose a Parse2Reason method containing two steps of first parsing then reasoning. 
In the parsing step, we utilize conditional description language (CDL), a formalized language composed of predicates specifically designed to construct geometric conditions, to represent both problem description (natural text) and solid diagrams (visual image).
In the reasoning step, we leverage those formal CDL and the theorem bank to perform relational inference and algebraic computation, generating strictly correct, verifiable, and human-readable reasoning processes.
Notably, our proposed Hilbert-Geo is also applicable to plane geometry.
To advance geometric reasoning, we curate two expert-annotated dataset SolidFGeo2k and PlaneFGeo3k, which are furnished with geometric formal language annotations, solutions and answers. Extensive experiments show that our proposed method achieves the state-of-the-art (SOTA) performance 77.3\% in SolidFGeo2k and 84.1\% in MathVerse-Solid (one small subset in MathVerse dedicated to solid geometry), substantially outperforming leading MLLMs, such as Gemini-2.5-pro (54.2\% on SolidFGeo2k) and GPT-5 (62.9\% on MathVerse-Solid). In addition, our method achieves the SOTA accuracy 80.2\% in PlaneFGeo3k, demonstrating the generality of the Hilbert-Geo in geometric reasoning. 
Our code and datasets are released at \href{https://github.com/PremiLab-Math/Hilbert-Geo}{https://github.com/PremiLab-Math/Hilbert-Geo}.
\end{abstract}

%% file: sec/1_intro.tex
\section{Introduction}
\label{sec:intro}
%The ability to reason about geometry has long been considered a hallmark of human intelligence. From the axiomatic foundations laid by Euclid to its modern applications in physics, engineering, and computer graphics, geometry provides a rich and structured domain for testing the limits of logical and spatial reasoning. In the field of Artificial Intelligence, automated geometry problem solving serves as a crucial benchmark, pushing systems beyond pattern recognition towards true symbolic understanding and deduction [2, 29, 47].

``Arithmetic symbols are written figures, and geometric figures are drawn formulas.''---David Hilbert\\
\begin{figure}[ht]
    \centering
    \includegraphics[width=1\linewidth]{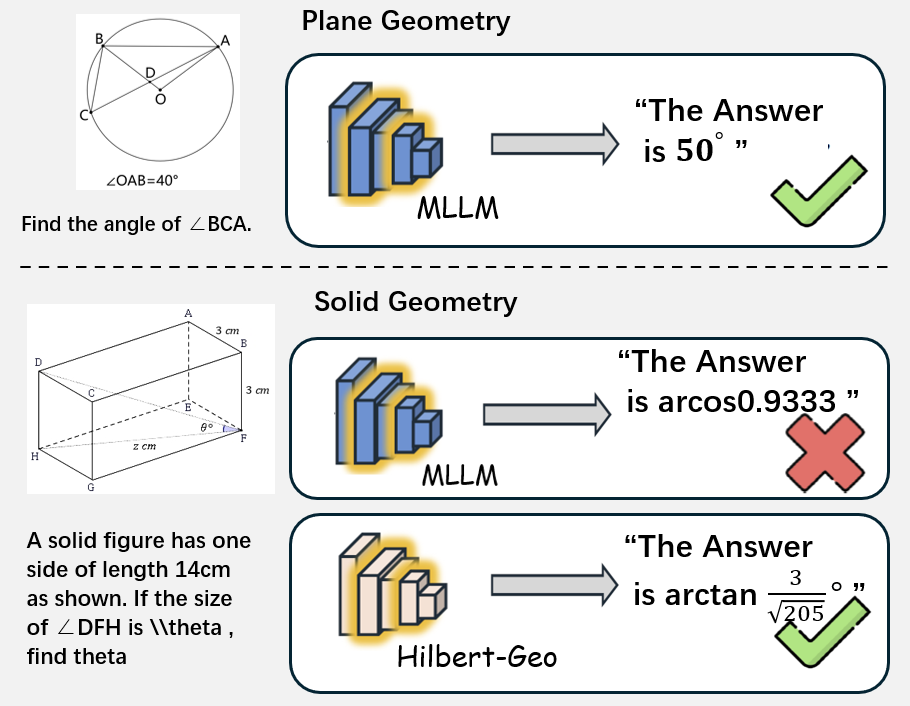}
    \caption{MLLMs struggle with problems in the field of solid geometry (middle) while showing performance in plane geometry problems (top). Our proposed Hilbert-Geo correctly solves this solid geometric problem (bottom).}
    \label{fig:1}
\end{figure}

In recent years, the AI community has made notable breakthroughs in plane geometry, with state-of-the-art systems proficient in solving routine problems and interpreting diagrams, supported by advances in formal reasoning and visual understanding~\cite{murphy2024autoformalizing}. This progress contrasts sharply with solid geometry, a far more complex and underexplored frontier (See Fig.\ref{fig:1}) that poses formidable hurdles for MLLMs, manifest in three key aspects:
First, solid geometry demands precise grasp of 3D spatial relationships and geometric core concepts~\cite{pittalis2010types}, requiring reasoning about occluded structures and spatial transformations. This stretches MLLMs’ ability to model abstract spatial logic~\cite{battaglia2018relational}, leading to reasoning errors and knowledge gaps~\cite{wang2025solidgeo}.
Second, its inherent multimodality requires sophisticated visual-linguistic grounding: MLLMs must parse text, interpret 2D representations of 3D objects, infer implicit spatial info, and align language with visuals while maintaining cross-modal consistency—often resulting in perception errors and hallucinations~\cite{vidhalluc2025}.
Third, MLLMs are prone to calculation errors in quantitative tasks, further undermining reliability.

\begin{figure}[htbp] % 位置参数：优先左栏内排版
    \centering
    % 左子图：占左栏宽度的 ~45%
    \begin{subfigure}{0.46\linewidth} % \linewidth 为当前栏宽（左栏）
        \centering
        \includegraphics[width=\textwidth]{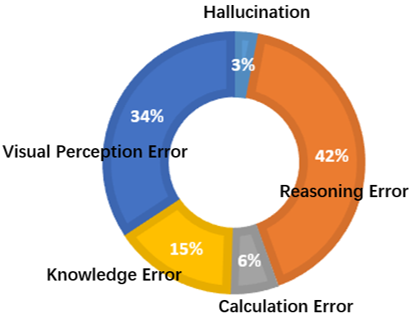} % 充满子图宽度
        \caption{Gemini 2.5 Pro}
        \label{fig:left1}
    \end{subfigure}
    \hfill % 两图间自动填充空白
    % 右子图：占左栏宽度的 ~45%
    \begin{subfigure}{0.5\linewidth}
        \centering
        \includegraphics[width=\textwidth]{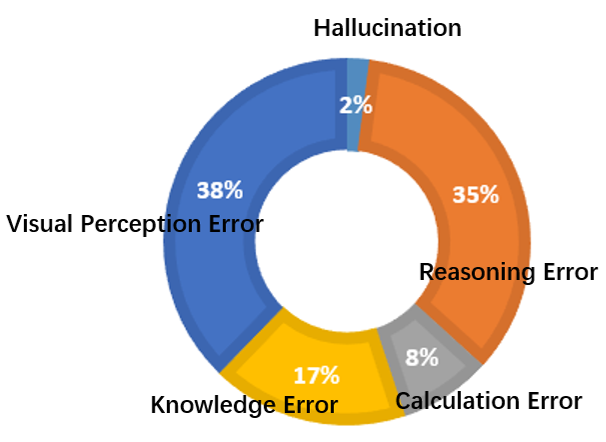}
        \caption{GPT-5}
        \label{fig:right1}
    \end{subfigure}
    \caption{ Error distribution in SolidFGeo2k of Gemini and GPT}
    \label{fig:two-in-left}
\label{error2}
\end{figure}
Existing solid geometry datasets and benchmarks such as MathVerse-Solid~\cite{zhang2024mathverse} and SolidGeo~\cite{wang2025solidgeo} are scattered and flawed. SolidGeo, with 3,000 samples, draws data from single exam question websites, posing potential data leakage risks. Its answers, generated by LLMs, are prone to undetected errors. Moreover, the dataset contains some math modeling and physics problems with low geometric relevance, failing to meet current research and application needs effectively. Additionally, most existing datasets only assess models based on final answers and are prone to data contamination~\cite{lu2023mathvista}.
To mitigate these limitations, we curate an expert-annotated dataset SolidFGeo2k to alleviate the scarcity of high-quality solid reasoning resources and establish a critical benchmark. This dataset is furnished with geometric formal language annotations, detailed reasoning path annotations, and accurate answer labels.

As shown in Fig.~\ref{error2}, we conducted a rigorous fine-grained error analysis on two representative models (Gemini 2.5 Pro~\cite{google2025gemini25} and GPT-5~\cite{openai2025gpt5})based on their performance on the SolidFGeo2k dataset. Specifically, we scrutinized all problems that the two models failed to solve correctly, categorizing the resulting errors into five distinct types~\cite{pittalis2010types}.

\begin{figure}[ht]
    \centering
    \includegraphics[width=1\linewidth]{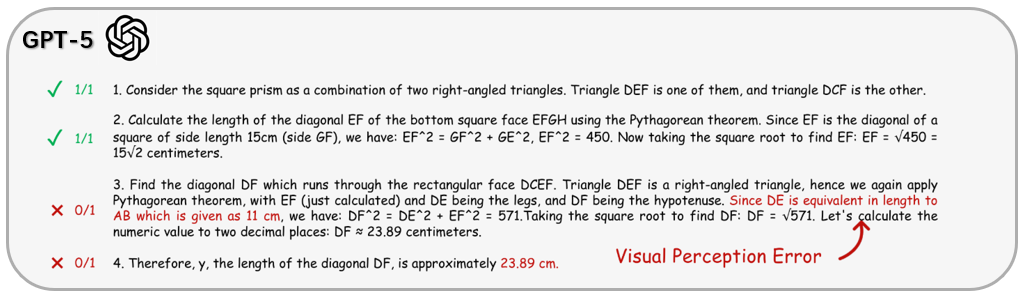}
    \caption{Visual Perception Error for the Problem in Fig.~\ref{fig:M2f}}
    \label{fig:placeholder1}
\end{figure}

Among all error categories, visual perception errors and reasoning errors are the predominant contributors. For Gemini 2.5 Pro, these two error types collectively account for 76\% of the total erroneous cases; for GPT-5, their combined proportion reaches 73\%. GPT-5 demonstrates a lower reasoning error rate (35\%) in comparison to Gemini 2.5 Pro (42\%), while Gemini 2.5 Pro exhibits a lower visual perception error rate (34\%) than GPT-5 (38\%) (See Fig.~\ref{fig:placeholder1}, an example of visual perception error).

To address these inherent limitations, this study fuses the perceptual capabilities of modern neural networks with the rigor of formal logic~\cite{hasan2015formal}. Building on the foundational theoretical framework of FormalGeo, a plane geometry formalization framework~\cite{zhang2023formalgeo}, we propose Hilbert-Geo: an explicitly structured formally structured framework~\cite{li2024survey} designed to tackle the unique challenges of solid geometry. This framework overcomes the core technical bottlenecks intrinsic to three-dimensional reasoning, laying a more robust foundation for solving solid geometric problems.

Comprehensive experimental evaluations demonstrate that Hilbert-Geo achieves state-of-the-art (SOTA) performance, attaining 77.3\% accuracy on the SolidFGeo2k benchmark and 84.1\% on MathVerse-Solid. This performance significantly outperforms leading multimodal large language models (MLLMs), including Gemini-2.5-Pro (which scored 54.2\% on SolidFGeo2k) and GPT-5 (with 62.9\% accuracy on MathVerse-Solid).

In summary, our primary contributions are fourfold:
\begin{itemize}
\item \textbf{Solid Geometry Formal Framework.} We develop a Solid Geometry Predicate Library for precise representation of entities and spatial relationships, complemented by a Solid Geometry Theorem bank.%formalizing hundreds of axioms and definitions via novel Geometric Predicates Logic.
\item \textbf{Multimmodal Formalization Parser.} We design a pipeline that parses both natural language and visual diagrams, resolving cross-modal ambiguities to ground textual entities in visuals, and translate problems into formal geometric language.
\item \textbf{Geometry Reasoning Engine.} We introduce an engine for inference on formalized problems, enabling an efficient search of verifiable and human-readable solutions.
\item \textbf{Two Geometric Datasets.} We construct SolidFGeo2k (solid geometry) and PlaneFGeo3k (plane geometry), both with meticulous formal annotations.
\end{itemize}

%% file: sec/2_formatting.tex
\section{Related Work}
\label{sec:formatting}
\subsection{Multimodal LLMs in Stereometric Reasoning}
While recent years have seen remarkable progress in LLMs and MLLMs, with strong performance across tasks like natural language understanding, visual question answering~\cite{zhao2023survey, hadi2023survey, wu2024survey, jiang2024survey, naveed2023comprehensive,chen2024overthinking,wang2025underthinking}, and multimodal mathematical reasoning (e.g., GPT-5~\cite{openai2024hello}, Gemini~\cite{gemini2023} surpass average human performance on MathVista \cite{lu2023mathvista}), they face critical limitations in stereometric reasoning~\cite{jian2023solving}. First, stereometric knowledge is heterogeneously disseminated across text, graphics, and intuition without a unified formal paradigm~\cite{wu2025nesygeo}, yet models rely on informal natural language, leading to ambiguity and failure to encode 3D entities' topological and metric details. Second, existing benchmarks focus mostly on plane geometry~\cite{ning2023symbolic, ning2025gns}, neglecting solid geometry's unique spatial and deductive demands. Consequently, LLMs/MLLMs struggle to deliver reliable stereometric reasoning and computations, highlighting an urgent need for targeted advancements.
\subsection{Geometric Reasoning with Formalization}
The quest for automated geometric problem solving has a rich history. Early systems—Wu's Method and subsequent algebraic approaches (e.g., Gröbner bases~\cite{hibi2003groebner}, cylindrical algebraic decomposition~\cite{arnon1984cad})—offered powerful mechanisms: they solved geometric problems by transforming them into algebraic equations.

Recent work has shifted to formalizing geometric problems for AI. Geometry3K~\cite{lu2021intergps} translates problems into formal statements. GeoQA~\cite{chen2021geoqa} and later geometry reasoning systems such as UniGeo~\cite{chen2022unigeo} use program-like or expression-tree representations to model problem-solving steps~\cite{xia2025geox}. The FormalGeo system~\cite{zhang2023formalgeo} introduced a structured plane geometry formalization approach, with a comprehensive predicate library and theorem system. While these advances plane geometry formalization, their scope remained confined to the plane.

Solid geometric knowledge exhibits formal heterogeneity: it is scattered across textual descriptions, spatial graphics, and abstract intuition, lacking a unified representation paradigm. Meanwhile, spatial relationships are inherently complex, and informal descriptions often cause ambiguous interpretations and reasoning hallucinations—conflicting with mathematical rigor. Worsening these issues, current formalization methods (e.g., plane geometry formalization, general logical systems) cannot effectively capture the topological characteristics and metric relationships of three-dimensional entities (e.g., polyhedrons, spherical surfaces, complex geometries).

%% file: sec/3_finalcopy.tex
\section{Solid Geometry Formal Language}
\subsection{Review of Geometry Formalization Theory}
\label{gft}
Geometry Formalization theory ~\cite{zhang2023formalgeo} is centered on the conversion of geometric problems, encompassing known conditions in natural language, problem goals, and geometric diagrams, along with their solution procedures into a unified and precise formal language. It comprises two key components: geometry ontology, which synthesizes core geometric concepts and their interconnections through a knowledge graph focused on Euclidean plane geometry, serving to guide the design of formal systems~\cite{badescu2004projective}; geometry representation theory, which investigates the formal expression of geometric knowledge within the framework of consistency theory, guaranteeing the coherence of both components (e.g. formalized diagrams) and processes (e.g. formal descriptions).%The development of this theory has been practically implemented in systems like FGeo-DRL, which integrates deductive reasoning with deep reinforcement learning for geometric problem solving ~\cite{zou2024fgeo}.
\subsection{Solid Geometry Formalization Theory}
The extension of the formal geometric theoretical framework to the domain of solid geometry is predominantly accomplished via the augmentation of geometric primitives and generalized combinatorial strategies. Furthermore, the original plane geometric formal category is encapsulated as an entity designated "plane", which serves as the foundational reference for constructing solid geometric models. Specifically, augmented geometric primitives enable the precise description of basic solid components, while generalized combinatorial strategies include the systematic integration of coverage techniques, clipping techniques, and topological connections, among others. This hierarchical extension strategy not only preserves the completeness of the original plane geometric formalization but also endows the framework with the capability to address complex three-dimensional geometric problems, thereby laying a solid theoretical foundation for subsequent research on solid geometric modeling and automated reasoning.
%geometry formal language, structured into geometry definition language (gdl)—incorporating predicate and theorem definitions to configure solvers, thereby enhancing shareability and extensibility—and condition declaration language (cdl)—employed to input problems through construction, condition, and objective statements that characterize specific geometric problems.
More details in Appendix~\ref{app:a1}.
\subsection{Predicate Library and Theorem Bank}
Predicates act as the fundamental building blocks for characterizing geometric entities. They are integrated into the Geometric Conditional Description Language (CDL), which not only eliminates ambiguities across multimodal representations but also establishes a robust premise for subsequent reasoning tasks ~\cite{zhang2023formalgeo}. 

Further elaborations are provided in Appendix~\ref{app:a2}. The Solid Predicate Library encompasses 120 predicates, comprising 35 fundamental predicates natively integrated into the solver, along with 20 entity types, 35 entity relationships, and 30 attribute descriptors defined via a dedicated predicate definition language. The formal specification for each predicate includes its name, declaration of point variables, validity check assertions, support for multiple representations, and mechanism for automatic expansion.

Theorems function as the cornerstone of reasoning processes, enabling search-based inference through the self-expansion of geometric entities and the application of theorem-driven deduction ~\cite{lu2021intergps}. Notably, when defining attributes, symbolic form declarations are also incorporated. The Theorem Bank is an extension of the plane FormalGeo framework, structured into two components: premises and conclusions (detailed in the accompanying table), and encompasses 220 theorems.
 

%% file: sec/4.tex
\section{Datasets of SolidFGeo2k and PlaneFGeo3k}
\begin{figure*}[ht]
    \centering
    \includegraphics[width=0.9\linewidth]{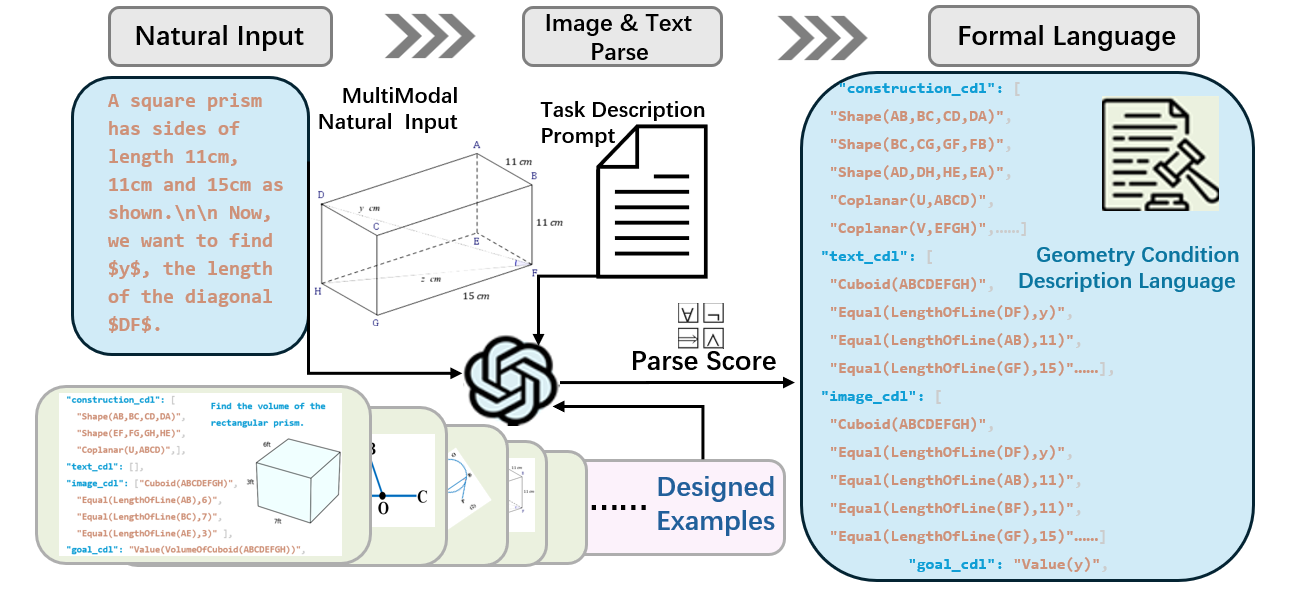}
    \caption{Overall process of parsing images and text into geometry condition description language (CDL)}
    \label{fig:M2f}
\end{figure*}
The formalization of natural language geometric descriptions into structured languages like CDL relies heavily on high-quality annotated datasets and benchmarks—serving as both the foundation for model guidance and the benchmark for performance evaluation. Among these, the formally annotated SolidFGeo2k, PlaneFGeo3k datasets are core to this work, providing standardized geometric content, multimodal resources, and CDL formalizations. Below, we first detail these datasets.Both datasets are manually annotated by domain experts to ensure accuracy and consistency, covering complementary geometric scenarios (solid and plane) and providing a comprehensive basis for parsing tasks.
\subsection{Dataset Collection}
SolidFGeo2k and PlaneFGeo3k are constructed to focus on geometric problem solving, covering two core domains of geometric reasoning. SolidFGeo2k comprises 1,908 stereometric problems centered on geometric solution derivation, while PlaneFGeo3k consists of 3,022 plane geometric problems. Both datasets adhere to a unified raw content collection framework, capturing three essential components for each problem.

First, problem descriptions are collected as structured natural language text, precisely conveying geometric configurations, given conditions, and solution targets. These descriptions are sourced from other datasets and renowned educational~\cite{chen2021geoqa} and research resources~\cite{wang2025mv}, ensuring alignment with practical geometric reasoning needs and encompassing a variety of distinct problem types~\cite{pittalis2010types}. Second, geometric images are paired with each problem description, visually representing the spatial relationships, shape compositions, and key elements of the geometric scenario~\cite{zhao2025pigps}. Third, standard answers are compiled for each problem, providing numerical results (e.g., volume, area, length, angle) or definitive descriptive conclusions (e.g., spatial relation confirmations) derived from standard geometric solution processes. These answers serve as the ground truth for validating problem-solving correctness.
The collection process prioritizes authenticity and diversity, ensuring the raw content reflects real-world geometric problem characteristics. More details are provided in Appendix~\ref{app:b}.
\subsection{Manual Annotation}
To transform raw problem content into actionable resources for geometric formalization and model training,datasets undergo systematic manual annotation by experts.All experts have experience in formal mathematical annotation. The annotation work follows standardized protocols and leverages professional tools.
Formal language encoding is implemented to translate the annotated information into geometric description language (CDL). This encoding unifies text and image semantics into a standardized structured format. The encoding process adheres to strict syntax rules to maintain consistency across the dataset.
To ensure annotation quality, a two-stage validation mechanism is adopted: 20\% of entries are randomly selected for cross-validation by an independent expert team, and inter-annotator agreement is measured using Cohen’s Kappa~\cite{cohen1960coefficient}. SolidFGeo2k achieves a Kappa score of 0.89, while PlaneFGeo3k reaches 0.91—both indicating "almost perfect" consistency . Final adjustments are made to resolve discrepancies, ensuring the annotations meet high standards of accuracy and reliability for downstream geometric reasoning tasks.

%% file: sec/5.tex
\section{Parse2Reasoning Method}
\subsection{MultiModal Formalization Parser (M2FP)}
\label{pa}
\begin{figure*}
    \centering
    \includegraphics[width=1\linewidth]{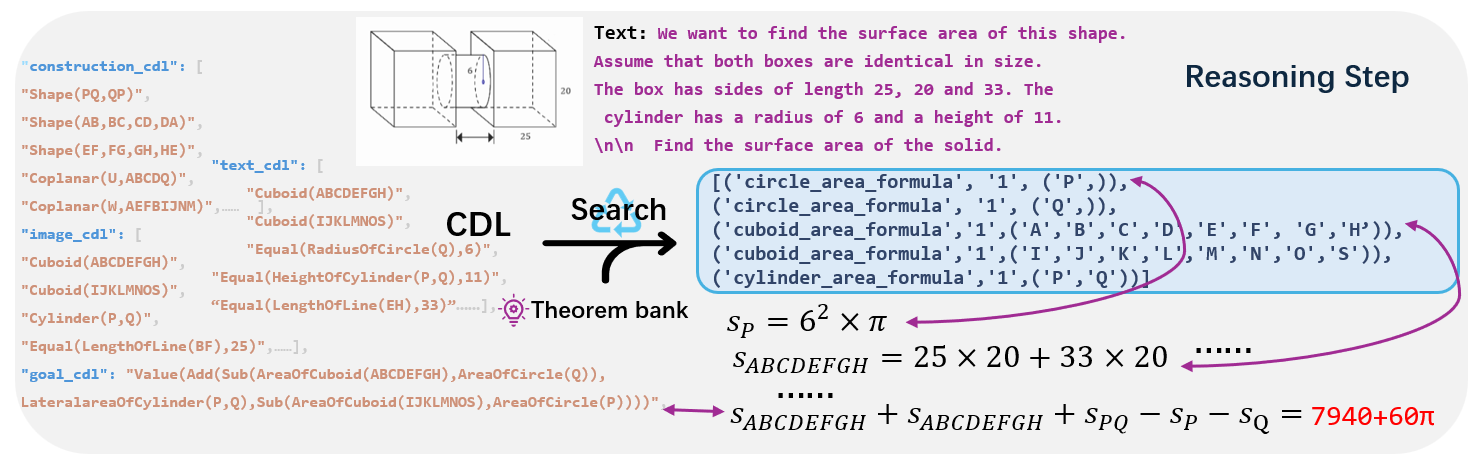}
    \caption{Reasoning using geometric condition description language (CDL)}
    \label{fig:r}
\end{figure*}
The translation of natural language geometric descriptions into formal languages is a foundational step in enabling automated geometric reasoning, as structured formalizations allow for precise computation of spatial relationships, property deductions, and theorem verification. This process involves two core stages: parsing (the autoformalization process of converting natural language to formal language) and reasoning (leveraging formal language for inference). To ensure robustness, both stages require rigorous evaluation, with parsing quality directly impacting downstream reasoning performance. Below, we outline the pipeline structure, evaluation methods, and experimental design to assess this workflow. More details in Appendix~\ref{app:c1}.

As shown in Fig.~\ref{fig:M2f}, this process starts with a targeted prompt engineering approach: specialized prompts are constructed using expertly designed example sets that cover all basic predicates. These examples can express predicates and formal language without involving complex geometric relationships, and each includes geometric descriptions and their corresponding formalizations. The model is guided to parse unseen geometric content into formal language using known predicates and formalized language, with these examples serving as structural and logical references.These examples will be presented to MLLMs along with a task description prompt of several hundred words to obtain the formal language (i.e., Condition Description Language).
%\subsection{Evaluation Method for Parsing Performance}
%To evaluate parsing efficacy, we employ the F1 score—a metric that balances precision (proportion of correctly generated CDL elements among all outputs) and recall (proportion of ground-truth CDL elements captured by the outputs)—to assess the alignment between generated CDL and dataset annotations. This quantifies both syntactic fidelity (adherence to CDL syntax rules) and semantic proximity (consistency with geometric intent). F1 scores are computed across all examples, enabling performance comparison of parsing under different in-context guidance scales.More details in Appendix C.2.
%Following evaluation, the parsed CDL outputs—filtered for sufficient F1 scores to ensure baseline quality—are fed into a downstream geometric reasoning module. This module leverages the formalized CDL language to perform tasks such as theorem verification, spatial relationship inference, and geometric property deduction, building on the structured representations generated during parsing.

%% file: sec/6.tex
\subsection{Solid Geometry Reasoning Engine (SGRE)}
\label{pp}
Building on the Multimodal Formalzation Parser, we introduce the Solid Geometry Reasoning Engine (SGRE), which is a dedicated module that drives inference using a comprehensive library of solid geometry predicates and theorems. SGRE focuses on transforming parsed CDL representations into actionable solutions through rigorous symbolic deduction, addressing the core challenge of multi-step geometric reasoning. Theoretically, provided that the problem itself is valid and the upstream parsing is accurate, a sound search mechanism will invariably yield a correct and verifiable solution process (See Fig.~\ref{fig:r}).

\begin{algorithm}
\caption{Theorems Search and Verification}
\begin{algorithmic}
\Require tree: a tree with the known problem conditions as the node.
\Ensure theorem\_seqs: list of theorem sequence for solving.
\State Initialize a list \textit{theorem\_seqs}
\State $node \gets tree.\text{get\_expandable}()$
\While{$node$ is not $\text{None}$}
    \State $solved \gets node.\text{apply\_theorem}()$
    \If{$solved$}
        \State $node.\text{state} \gets \text{SOLVED}$
        \State $\textit{theorem\_seqs} \gets node.\text{get\_theorem\_seqs}()$
        \State \textbf{break}
    \EndIf
    \State $node.\text{state} \gets \text{EXPANDED}$
    \State $l \gets node.\text{expand}()$
    \If{$l = 0$}
        \State $node.\text{state} \gets \text{UNSOLVED}$
    \EndIf
    \State $node \gets tree.\text{get\_expandable}()$
 \EndWhile
\end{algorithmic}
\label{algorithm1}
\end{algorithm}
%At its core, SGRE operates on a search framework rooted in a specialized solid geometry knowledge base. This knowledge base comprises two key components: a predicate library encoding 3D-specific spatial relations , shape attributes , and entity definitions; and a theorem library containing formalized geometric laws. 

Starting from the CDL output of the parser, SGRE initializes a node-based reasoning tree.(See Algorithm \ref{algorithm1}) Each node represents a geometric state: expandable nodes correspond to unresolved sub-problems, solved nodes mark completed sub-goals, and failed nodes indicate invalid theorem applications. The engine iteratively matches expandable nodes against applicable theorems from the library, applying logical substitutions and calculations to generate new states,and computes the final result through chained deduction. 

Since SGRE addresses a NP problem, most failed cases arise from combinatorial explosion, which prevents finding a goal-satisfying solution within a reasonable time. Notably, over 80\% of solvable problems can be resolved within 1.2 seconds and 57 reasoning steps. The majority of unsolved cases stem from erroneous CDL outputs or inherent flaws in the problem itself, thereby enabling SGRE to also assess the validity of problem premises and solutions.

By evaluation on the SolidFGeo2k dataset, SGRE achieves an 78\% accuracy rate across complex tasks, validating the effectiveness of its predicate-theorem framework. By grounding inference entirely in structured solid geometric knowledge, SGRE ensures traceable, human-aligned reasoning, complementing the parser’s role in converting multimodal inputs to formal representations. Together, these components form a complete pipeline for automated solid geometry problem-solving. More details are in Appendix~\ref{app:d}.

%% file: sec/7.tex
\section{Experiments}
To validate the efficacy of the Hilbert-Geo framework in facilitating solid geometry reasoning and quantify its superior performance compared to state-of-the-art multimodal large language models (MLLMs), we have devised a suite of comprehensive experiments. In this section, we first elaborate on the experimental design and setup, followed by a detailed presentation of the quantitative results pertaining to both the parsing phase and the reasoning phase. Subsequently, we conducted an in-depth analysis of the empirical findings to elucidate the underlying mechanisms and practical implications. More details are provided in Appendix~\ref{app:e}.
%(This section discusses the SolidFGeo2k dataset; for more details in other datasets, see Appendix~\ref{app:e})
\subsection{Experimental Setting}
\textbf{Parsing step.} As shown in Fig.~\ref{fig:M2f}, the parsing step, as the core perceptual component of the Hilbert-Geo framework, is subjected to independent performance evaluation to decouple its intrinsic effectiveness from the interference of the subsequent reasoning step. To ensure the comprehensiveness and representativeness of the evaluation, we conducted systematic testing in five closed-source models and four open-source models, with four gradient sample quantities (15, 25, 35, and 45 samples respectively). In the parsing steps of Section \ref{pa}, these samples were sampled from expertly designed example sets that cover diverse task scenarios and data distributions.\\
\textbf{Reasoning step.} To comprehensively assess the performance of the Hilbert-Geo framework, we conduct five rounds of evaluations on a diverse set of multimodal large language models (MLLMs) across the benchmark datasets SolidFGeo2k, PlaneFGeo3k, and Mathverse-Solid. This evaluation covers 5 open-source and 4 closed-source models, ensuring a thorough comparison across different model types and access paradigms.

To establish a human performance baseline, we recruited 50 high school students to independently complete the closed-book questions. Furthermore, we feed the parsing results derived from Section~\ref{pp} into the Solid Geometry Reasoning Engine, which characterize the models’ perceptual ability to convert geometric inputs into formal CDL. Here, we assess the success rate of geometric reasoning, measuring how effectively each model executes deductive processes based on the parsed formal structures.
\subsection{Experimental Results}
%\subsubsection{Reasoning Performance}
\begin{figure}[ht]
    \centering
    \includegraphics[width=1\linewidth]{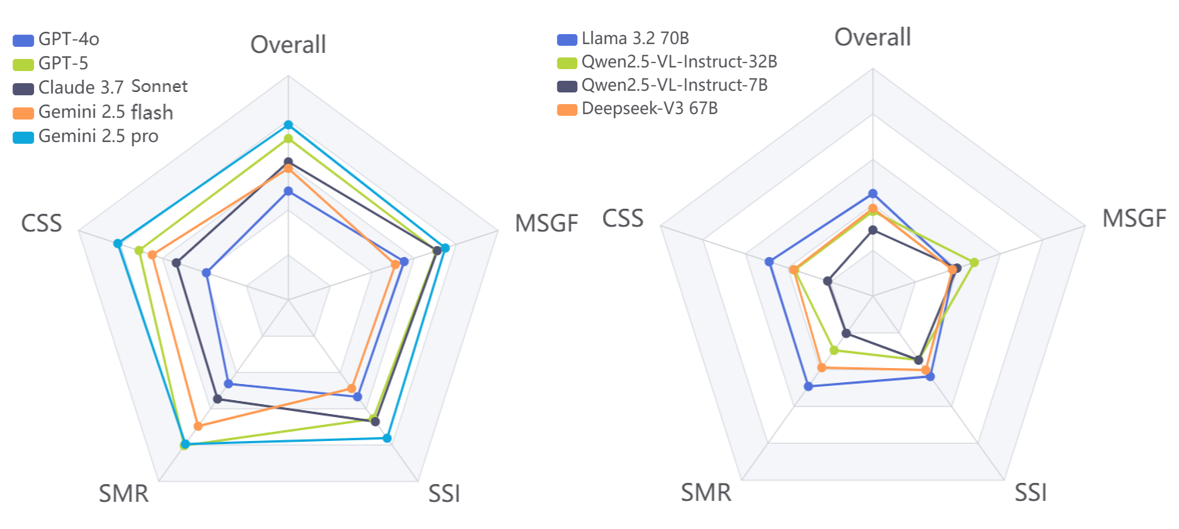}
    \caption{Performances for Different Subjects (CSS, SMR, SSI, MSGF, See Table \ref{tab:model_performance_trimmed}) based on Hilbert-Geo}
    \label{fig:lr}
\end{figure}
\begin{table}[ht]
\centering
\caption{MLLMs and Hilbert-Geo (Gemini-2.5-pro 45 samples) performances on SolidFGeo2k; Four fine grains: Composite Solid Structures (CSS), Spatial Metric Relations (SMR), Solid Shape Identification (SSI),  Measurement of Solid Geometric Forms (MSGF); The underline represents the best performance of MLLMs}
\resizebox{\linewidth}{!}{%
\begin{tabular}{l|ccccc}

\hline
\makecell{Model} & Overall.Avg & CSS & SMR & SSI & MSGF \\ \hline
\multicolumn{6}{c}{ Closed-source MLLMs} \\ 
\hline
\makecell{GPT-4o\cite{openai2024gpt4o}} & 35.8 & 25.9 & 27.8 & 39.2 & 40.0\\
\makecell{GPT-5\cite{openai2025gpt5}} & 50.6 & 51.2 & 55.4 & 40.6 & 46.8 \\ 
\makecell{Claude 3.7 Sonnet\cite{anthropic2025claude37}} & 47.1 & 39.2 & 40.4 & \underline{52.9} & \underline{56.8} \\
\makecell{Gemini 2.5 Flash\cite{google2025gemini25}} & 39.8 & 50.7 & 55.4 & 34.6 & 36.8 \\
\makecell{Gemini 2.5 Pro\cite{google2025gemini25}} & \underline{54.2} & \underline{60.2} & \underline{62.4} & 48.2 & 49.8 \\ \hline
\multicolumn{6}{c}{Open-source MLLMs} \\ \hline
\makecell{Llama 3.3 70B\cite{meta2025llama3.3}} & 33.6 & 36.3 & 34.4 & 31.1 & 28.7 \\
\makecell{Qwen2.5-VL-Instruct-32B\cite{bai2025qwen}} & 29.8 & 27.2 & 19.9 & 35.2 & 38.1 \\
\makecell{Qwen2.5-VL-Instruct-7B\cite{bai2025qwen}} & 20.2 & 17.7 & 10.7 & 25.2 & 30.1 \\
\makecell{Deepseek-V3 67B\cite{deepseek2024v3}} & 30.1 & 11.7 & 11.3 & 28.6 & 27.8 \\ \hline
\multicolumn{6}{c}{Human Performances} \\ \hline
\makecell{Human} & 81.8 & 84.3 & 81.2 & 86.1 & 78.7 \\ \hline
\multicolumn{6}{c}{Hilbert-Geo} \\ \hline
\makecell{Hilbert-Geo (Ground-Truth) (Ours)} & 78.7 & 80.5 & 84.1 & 76.3 & 75.1\\
\makecell{Hilbert-Geo (Gemini-2.5-Pro) (Ours)} & 77.3 & 80.3 & 79.4 & 76.2 & 74.8 \\
\hline
\end{tabular}%
}
\label{tab:model_performance_trimmed}
\end{table}

The findings in Table \ref{tab:model_performance_trimmed} underscore the inherent challenges of Solid Geometry. Gemini 2.5 Pro emerges as the top-performing model with an overall accuracy of 54.2\%, followed by GPT-5 at 50.6\% and Claude-3.7-Sonnet at 47.1\%. Notably, all other models score below 40\%, highlighting the significant hurdles that solid geometry reasoning presents even for cutting-edge large language models. 
\begin{figure}[ht]
    \centering
    \includegraphics[width=0.93\linewidth]{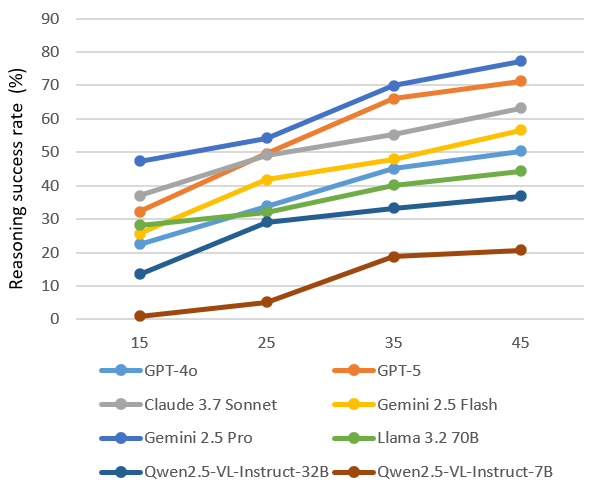}
    \caption{Reasoning performance of MLLMs under different numbers of samples based on Hilbert-Geo}
    \label{fig:tend}
\end{figure}

Across fine-grained tasks that demand robust reasoning, including Composite Solid Structures (CSS), Spatial Metric Relations (SMR), Solid Shape Identification (SSI), and Measurement of Solid Geometric Forms (MSGF), notable performance gaps persist among leading models. Gemini 2.5 Pro takes the lead in SMR with a score of 62.4\% but lags behind in SSI, achieving only 48.2\%. Open-source models, with Llama 3.3 70B as a representative example (scoring 36.3\% in CSS and 28.7\% in MSGF), consistently underperform compared to their closed-source counterparts. Critically, all these models remain significantly behind human performance, highlighting the current limitations in replicating human-level reasoning in solid geometry tasks.
Hilbert-Geo, however, presents a different scenario (See Fig.~\ref{fig:lr}). When evaluated on Hilbert (using 45 samples), Hilbert-Geo achieves the highest performance across all key dimensions: 80.3\% in CSS, 79.4\% in SMR, 76.2\% in SSI, 74.8\% in MSGF, and an overall score of 77.3\%. It outperforms leading models like Gemini 2.5 Pro and approaches human-level performance.
%Hilbert-Geo, however, presents a different scenario (See in Fig. \ref{fig:lr}). When evaluated on Hilbert (using 45 samples), Gemini 2.5 Pro delivers near-human performance across three key dimensions: CSS at 81.3\%, SMR at 79.4\%, and MSGF at 74.8/\%.

\subsection{Ablation study}
\subsubsection{Parsing Performance}
\label{subsec:fuzzy_metrics}
\begin{figure}[ht]
    \centering
    \includegraphics[width=1\linewidth]{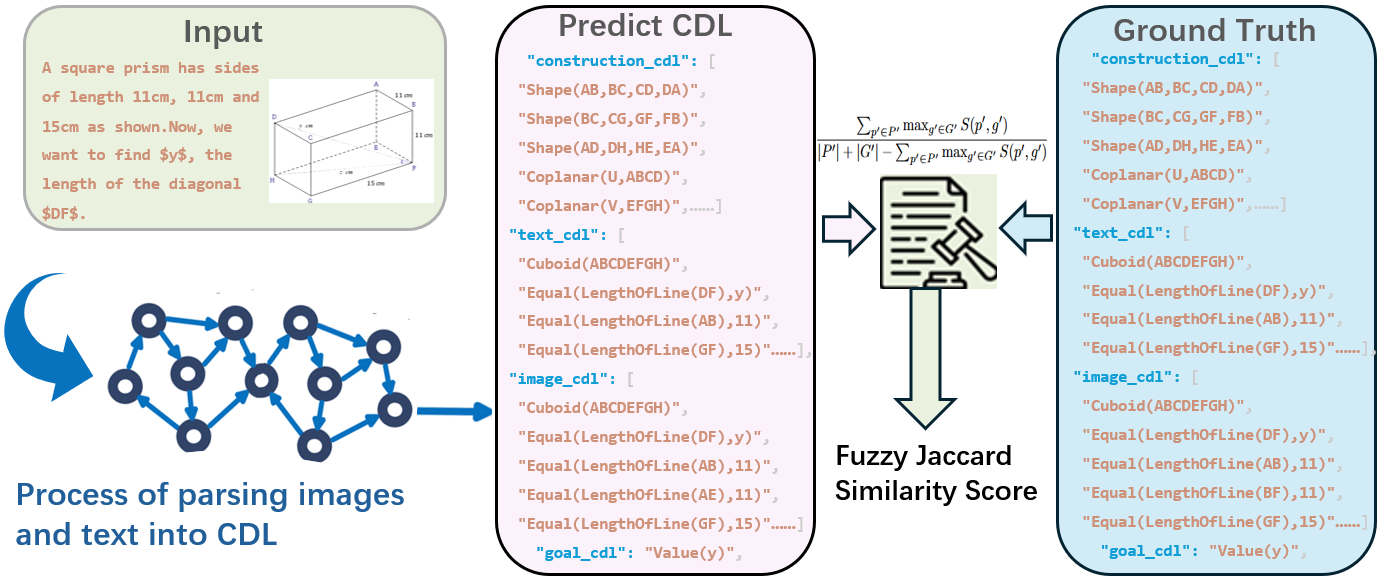}
    \caption{Fuzzy Jaccard Similarity Score between Predicted CDL and Ground Truth}
    \label{fig:fuzzy}
\end{figure}
As shown in Fig.~\ref{fig:fuzzy}, for the quantitative assessment of parsing performance, we adopted the Fuzzy Jaccard Similarity as the core accuracy metric. This metric introduces a fuzzy matching mechanism: it first computes the matching scores between the predict CDL results and the ground truth across multiple dimensions, then delineates the fuzzy intersection and fuzzy union based on these continuous matching scores, and finally quantifies the similarity between the two sets by calculating the ratio of the fuzzy intersection to the fuzzy union:\\
\begin{equation}
\text{Score} = \frac{\sum_{p' \in P'} \max_{g' \in G'} S(p', g')}{|P'|+|G'| - \sum_{p' \in P'} \max_{g' \in G'} S(p', g')}
 \label{}
\end{equation}\\
$P'$ is the normalized predicted CDL set; $G'$ is the normalized ground truth CDL set; $p'\in P'$; $g'\in G'$; $S(p', g')$ is the fuzzy matching scoring function; More details in Appendix~\ref{app:c2}.
\begin{figure}[ht]
    \centering
    \includegraphics[width=1\linewidth]{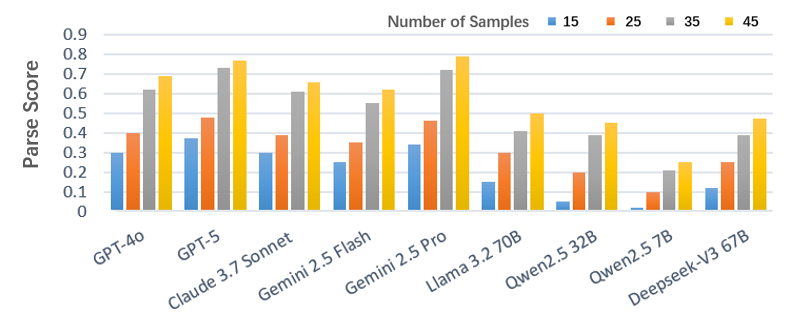}
    \caption{The parsing scores of various MLLMs under different sample quantities (More datasets' results in Appendix~\ref{app:c3}.)}
    \label{fig:parse}
\end{figure}

As shown in Fig.~\ref{fig:parse}, the experimental results validate that the parsing step of Hilbert-Geo can effectively distinguish the performance of different MLLMs, with GPT-5 and Gemini 2.5 Pro standing out as the most effective models for the parsing phase. Overall, they demonstrate the most robust performance, maintaining a score above 0.7 with 45 samples, which indicates their strong capability in parsing images and text into CDL within the formal framework. This lays a solid foundation for the subsequent reasoning step in the Hilbert-Geo framework, as a high-quality parsing output is crucial for accurate geometry reasoning.

\subsubsection{Impact of Various Models and Number of Samples}

\begin{table}[ht]
\centering
\caption{Performance of Different Models on Hilbert-Geo (45 Samples): Accuracy, Average Time per Solved Problem, and Reasoning Steps; The underline represents the best performance of Hilbert-Geo}
\resizebox{1\linewidth}{!}{%
\begin{tabular}{l|c c ccc|c|cc}
\hline
Model & Overall.Avg & CSS & SMR & SSI & MSGF & Avg.time & Avg.steps \\ \hline
\multicolumn{8}{c}{Closed-source MLLMs} \\ \hline
Hilbert-Geo (GPT-4o)\cite{openai2024gpt4o} & 50.3 & 39.1 & 46.2 & 53.4 & 55.2 & 38.1 & 112.6 \\
Hilbert-Geo (GPT-5)\cite{openai2025gpt5} & 71.2 & 71.2 & \underline{80.3} & 65.5 & 70.7 & 77.5 & 152 \\
Hilbert-Geo (Claude 3.7 Sonnet)\cite{anthropic2025claude37} & 63.2 & 53.4 & 54.6 & 67.1 & 71.0 & 50.0 & 129.2\\
Hilbert-Geo (Gemini 2.5 Flash)\cite{google2025gemini25} & 56.5 & 64.9 & 69.6 & 48.8 & 51.0 & 47.1 & 117.9 \\
Hilbert-Geo (Gemini 2.5 Pro)\cite{google2025gemini25} & \underline{77.3} & \underline{80.3} & 79.4 & \underline{76.2} & \underline{74.8} & \underline{81.2} & \underline{159.1} \\ \hline
\multicolumn{8}{c}{Open-source MLLMs} \\ \hline
Hilbert-Geo (Llama 3.3 70B)\cite{meta2025llama3.3} & 44.3 & 48.8 & 49.1 & 43.7 & 38.3 & 28.0 & 101.1 \\
Hilbert-Geo (Qwen2.5-VL-Instruct-32B)\cite{bai2025qwen} & 36.9 & 36.8 & 29.5 & 34.8 & 47.7 & 14.2 & 76.9  \\
Hilbert-Geo ((Qwen2.5-VL-Instruct-7B)\cite{bai2025qwen} & 30.7 & 21.3 & 20.3 & 34.8 & 39.6 & 2.9 & 32.0  \\
Hilbert-Geo ((Deepseek-V3 67B)\cite{deepseek2024v3} & 38.8 & 37.3 & 38.9 & 40.2 & 37.4 & 22.9 & 90.3 \\ \hline
\multicolumn{8}{c}{Formal Language Ground Truth} \\ \hline
Hilbert-Geo (CDL) & 78.7 & 80.5 & 84.1 & 76.3 & 75.1 & 88.0 & 174.6 \\ \hline
\end{tabular}%
}
\label{tab:model_performance_solidfgeo_avg_front}
\end{table}
%\textbf{Formal Framework Hilbert-Geo Performances}
\begin{figure}[ht]
    \centering
    \includegraphics[width=1\linewidth]{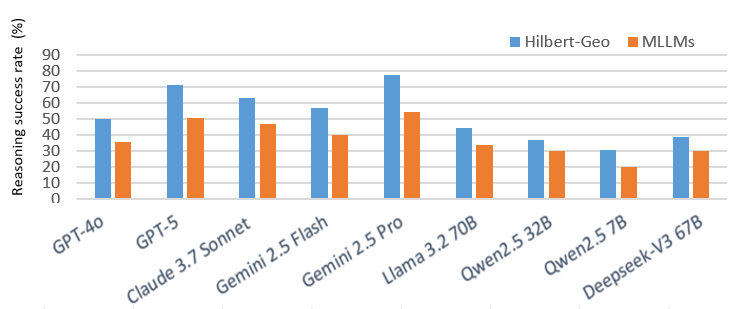}
    \caption{Comparison of Reasoning Performances between Hilbert-Geo (45 samples) and MLLMs}
    \label{fig:placeholder2}
\end{figure}

As shown in Table \ref{tab:model_performance_solidfgeo_avg_front} and Table \ref{tab:model_performance_trimmed}, it achieves a remarkable qualitative breakthrough in reasoning, boasting an average reasoning accuracy of 77.3\% in datasets with a scale 1.5 times that of pure MLLMs (See Fig.~\ref{fig:placeholder2}). Even with the smallest sample size of 15 examples, its accuracy in SolidFGeo2k reaches 42.1\% (See Fig.~\ref{fig:tend}), surpassing Deepseek-V3 67B’s peak performance.

Across pivotal benchmarks, Hilbert-Geo attains 77.3\% accuracy on SolidFGeo2k and consistently outperforms state-of-the-art MLLMs and specialized solvers. This superiority originates from two core strengths of the Hilbert-Geo framework: firstly, the CDL effectively encapsulates both the intrinsic geometric properties and the problem-specific conditions, while the SGRE conducts rigorous formal reasoning to derive the final solution. This symbol-driven approach ensures the reasoning process is rigorous and logically coherent; Secondly, the framework effectively steers MLLMs’ geometric intuition during the parsing stage, guaranteeing that initial geometric conditions are unambiguous, precise, and devoid of hallucinations.

Additionally, from Fig.~\ref{fig:tend}, we can observe that models with better performance have longer inference times and more steps. This is because the CDL complexity after parsing difficult questions is high, while the search in the inference process is fixed (always exponential complexity). Therefore, the average time and the average steps also indirectly indicate the parsing ability of the model.

\subsubsection{Error Analysis}

\begin{figure}[htbp] % 位置参数：优先左栏内排版
    \centering
    % 左子图：占左栏宽度的 ~45%
    \begin{subfigure}{0.48\linewidth} % \linewidth 为当前栏宽（左栏）
        \centering
        \includegraphics[width=\textwidth]{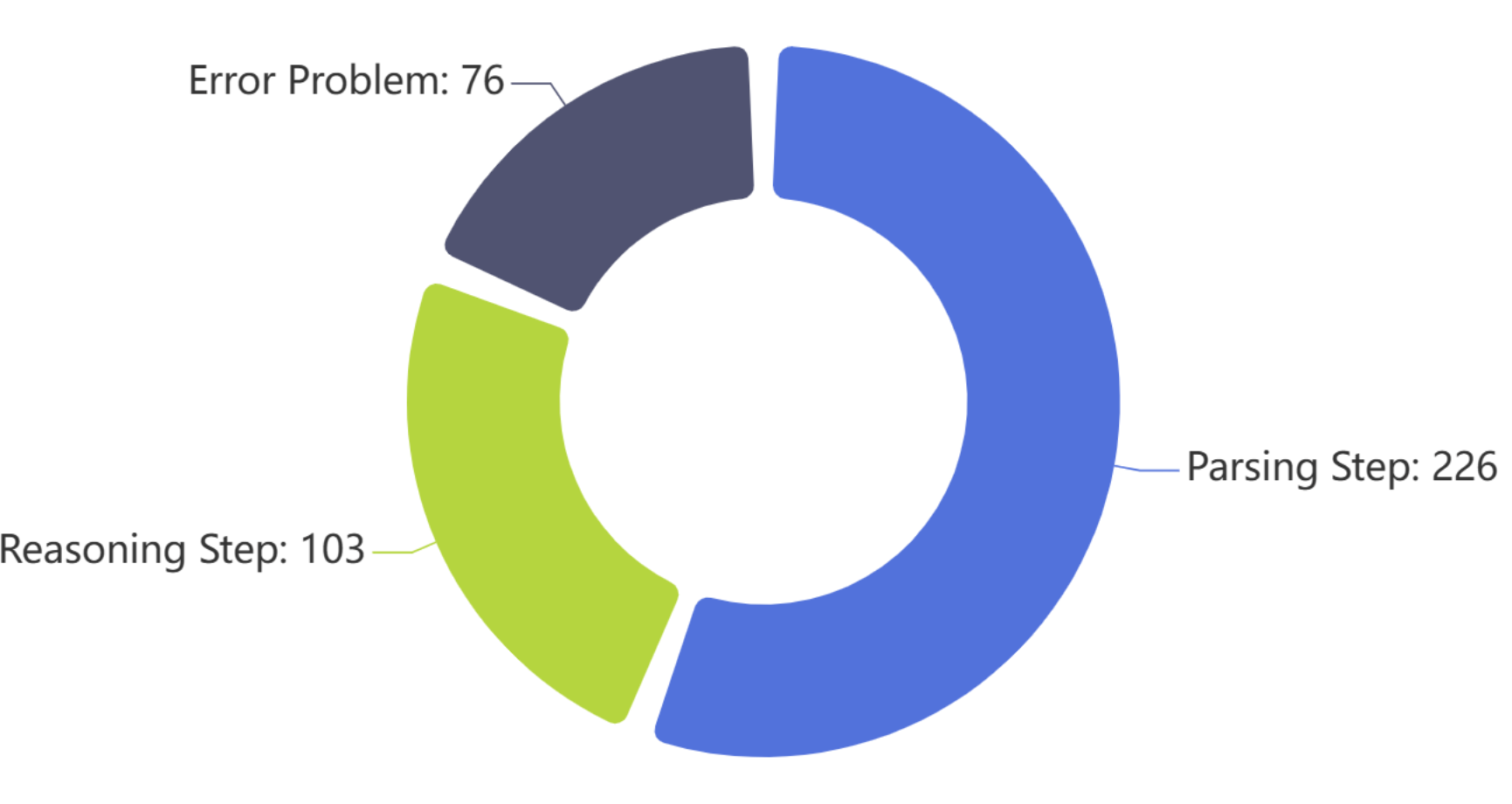} % 充满子图宽度
        \caption{Gemini 2.5 Pro}
        \label{fig:left}
    \end{subfigure}
    \hfill % 两图间自动填充空白
    % 右子图：占左栏宽度的 ~45%
    \begin{subfigure}{0.5\linewidth}
        \centering
        \includegraphics[width=\textwidth]{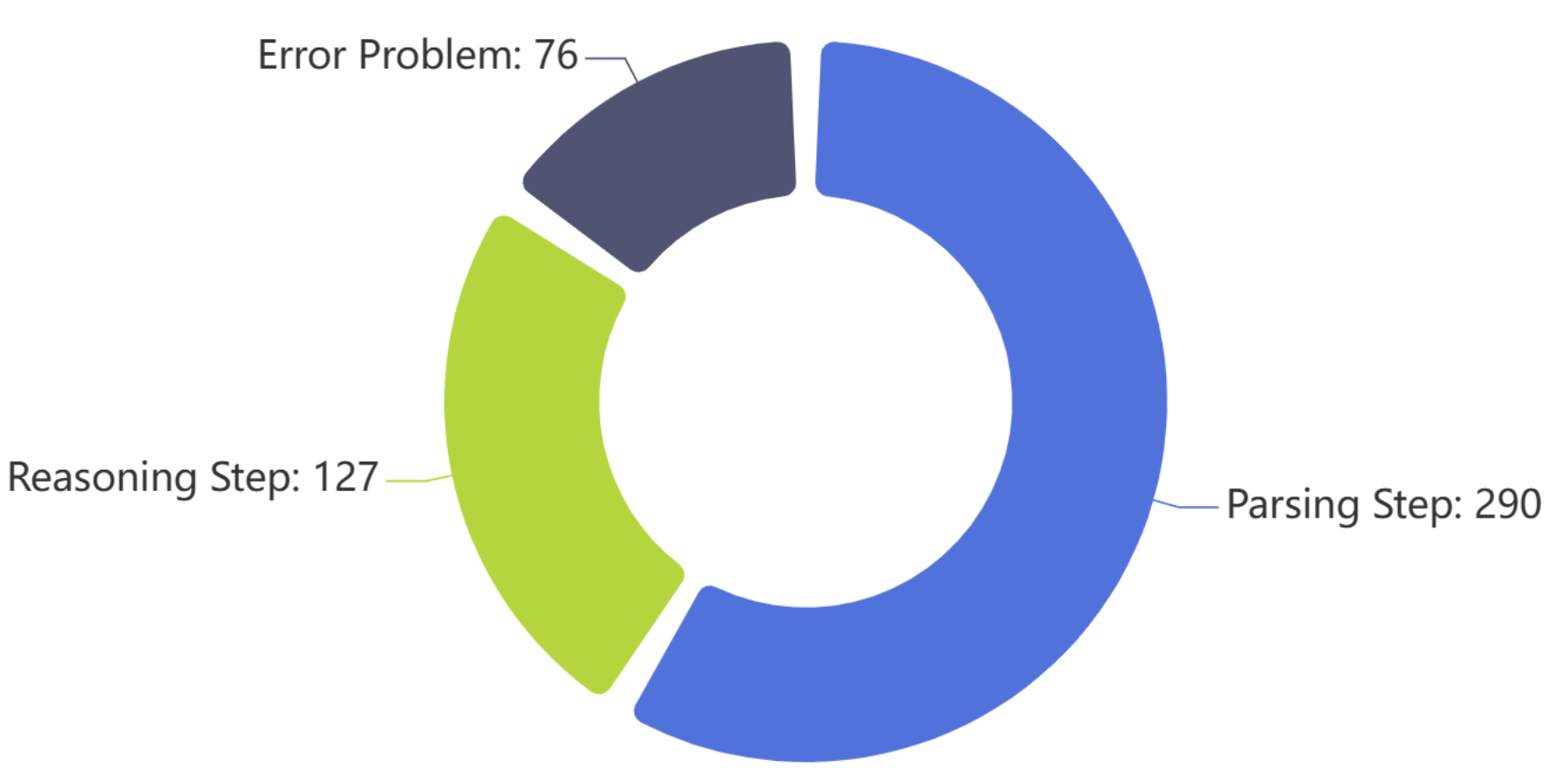}
        \caption{GPT-5}
        \label{fig:right}
    \end{subfigure}
    \caption{ Error distribution of Hilbert-Geo (Gemini 2.5 pro and GPT-5 with 45 samples)}
    \label{fig:two-in-left2}
\label{error1111}
\end{figure}
\begin{figure}
    \centering
    \includegraphics[width=1\linewidth]{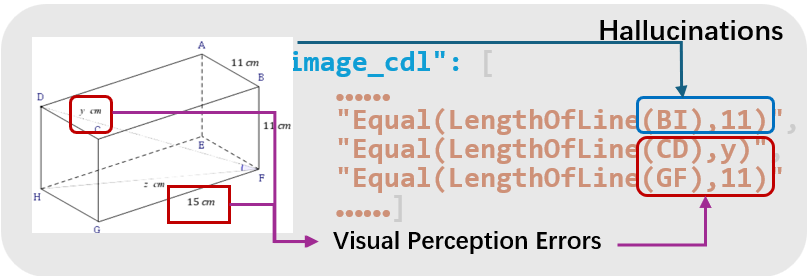}
    \caption{example of hallucination and Visual Perception Error}
    \label{fig:placeholdererror}
\end{figure}
As illustrated in Fig.~\ref{error1111}, the errors occurring in Hilbert-Geo primarily stem from three key sources: CDL distortion during the parsing phase, combinatorial explosion induced by excessively complex problem combinations in the reasoning phase, and inherent defects in the problem conditions themselves.
For CDL distortion in the parsing phase, we performed two targeted analyses on CDL, focusing on hallucinations and Visual Perception Errors (See Fig.~\ref{fig:placeholdererror}). The results indicated a substantial reduction in both the incidence of hallucinations and Visual Perception Errors, which fully demonstrates that the parsing phase of Hilbert-Geo can effectively alleviate these two types of issues.
In terms of the reasoning phase, we adjusted the time limit setting by removing the original 300-second constraint, we tested 103 problems that were prone to timeouts under the former setting; 87\% of these problems were successfully solved within the tested time frame, offering support for the completeness of Hilbert-Geo’s knowledge base. Theoretically, Hilbert-Geo exhibits no calculation and reasoning errors throughout the reasoning process.
%Regarding inherent defects in the problem conditions, we identified 76 problematic cases by examining the "Unsolved" category. For each of these cases with flawed conditions, we either implemented necessary corrections to rectify the defects or removed the problematic conditions entirely, thereby further improving the overall reliability of Hilbert-Geo.

%As shown in figure \ref{error1111}, Hilbert-Geo errors are mainly caused by CDL distortion during the parsing phase, combinatorial explosion from overly complex problem combinations in the reasoning phase, or inherent issues with the problem conditions themselves.\\
%For CDL distortion in parsing, we conducted two types of analysis on CDL: hallucinations and Visual Perception Errors. The results showed a significant reduction in both hallucination types and Visual Perception Errors, indicating that the parsing phase effectively mitigates these two issues.\\
%Regarding the reasoning phase, we removed the original 300-second time limit for reasoning (which previously led to timeout issues) and adopted an unlimited timeout setting. Testing on 103 timeout-prone problems revealed that 87\% of them were solved within the tested time frame, demonstrating the completeness of Hilbert-Geo’s knowledge base.Hilbert-Geo will not have any computational errors.\\
%For inherent problems with the problem conditions, we identified 76 problematic cases through the "Unsolved" category, and each of these flawed conditions was either corrected or removed.

%% file: sec/9.tex
\section{Conclusion}
In this paper, we introduce Hilbert-Geo, the first unified solid geometric formal framework to tackle the unique challenges of automated solid geometric reasoning, aiming to address a long-standing gap in AI systems’ ability to handle spatial logic and multimodal 3D diagram interpretation. By integrating a multimodal formalization parser and a solid geometry reasoning engine, which are tailored to capture 3D entity relationships and mitigate spatial hallucinations common in MLLMs, the framework fills this critical gap. Extensive experiments conducted on our curated SolidFGeo2k, PlaneFGeo3k, and MathVerse-solid datasets validate that Hilbert-Geo outperforms leading MLLMs, as it can effectively alleviate reasoning hallucinations, ensure logical consistency, and guide geometric intuition through formalized representations. Additionally, the framework targets the limited capability of existing methods in addressing vector-related tasks and problems requiring sophisticated auxiliary line constructions, as these gaps further underscore the need for a specialized approach. Furthermore, error analysis reveals that visual perception and logical reasoning deficits are the primary bottlenecks for current MLLMs in solid geometry, which highlights the value of Hilbert-Geo’s structured approach in mitigating these limitations.

%% file: sec/aA.tex
\clearpage
\setcounter{page}{1}
\maketitlesupplementary

\appendix
% 关键：配置附录计数器，避免 .1. 格式
\renewcommand{\thesection}{\Alph{section}} % 一级附录用大写字母（A、B、C...）
\renewcommand{\thesubsection}{\thesection.\arabic{subsection}} % 二级附录用 A.1、A.2...
%\renewcommand{\thetheorem}{\thesection.\arabic{theorem}} % 定理编号关联附录（A.1、A.2...）
%\renewcommand{\thedefinition}{\thesection.\arabic{definition}} % 定义编号关联附录
%\renewcommand{\thelemma}{\thesection.\arabic{lemma}} % 引理编号关联附录

% --- START OF APPENDIX CONTENT ---

% 一级附录：用 section（原 subsection 改为 section）
\section{Solid Geometry Formal Language}
\label{app:a}

% Define theorem environments（计数器关联 section，确保编号正确）
\newtheorem{theorem}{Theorem}[section]
\newtheorem{definition}{Definition}[section]
\newtheorem{lemma}{Lemma}[section]

% Define the custom operator for 3D composition
\newcommand{\oplusThreeD}{\oplus_{3D}}

% 二级附录：用 subsection（层级正确）
\subsection{Formal Geometry Representation}
\label{app:a1}
In the domain of solid geometry, simple geometric bodies serve as fundamental units that are combined into complex geometries via combinatorial strategies. This section provides a formal proof that the solid construction operation constitutes a commutative semigroup.

Previous work on Geometry Formalization Theory has shown that the formal representation of plane geometry is a semigroup (Section. \ref{gft}). In solid geometry, we elevate this representation to a set of faces. The plane geometric category is encapsulated as an entity, which serves as the foundational reference for constructing solid geometric models.

Let the formal representation of a simple polyhedron $\mathcal{P}$ be a set $R_P$ (Here, the polyhedron refers to a figure that is topologically homeomorphic to a $\mathbb{R}^3$ sphere):
\begin{equation}
    R_P = \{ f^{(P)}_1, f^{(P)}_2, \dots, f^{(P)}_m \}
\end{equation}
where each face $f^{(P)}_i$ is an ordered list of vertices (following the right-hand rule with outward normal vectors) describing the boundary of that face:
\begin{equation}
    f^{(P)}_i = (v_1, v_2, \dots, v_k)
\end{equation}

\paragraph{Definition of Composition Operation ($\oplusThreeD$):}
This is the foundation of the combined strategy, let polyhedra $A$ and $B$ be joined via a shared interface. Let $S_{A}$ denote the shared face in $A$, and $S_{B}$ denote the corresponding shared face in $B$ (note: geometrically, the vertex order of $S_{A}$ is the reverse of $S_{B}$).
The solid composition operation is defined as the set union minus the symmetric intersection (the internal contact faces), as shown in Eq. \ref{eq:composition_def}:
\begin{equation}
    \label{eq:composition_def}
    R_A \oplusThreeD R_B = (R_A \setminus \{S_{A}\}) \cup (R_B \setminus \{S_{B}\})
\end{equation}

\begin{theorem}
The formal representation of solid geometry is closed under the operation $\oplusThreeD$.
\end{theorem}

\begin{proof}
By definition, a valid object in the solid formal system is a "set of faces bounding a closed 3D space".
\begin{enumerate}
    \item Let $R_A$ and $R_B$ be face sets of two valid polyhedra.
    \item Execute the operation $R_{result} = R_A \oplusThreeD R_B$.
    \item This operation removes the internal contact faces ($S_A$ and $S_B$) and retains all external surfaces.
    \item According to the generalization of Euler's formula for manifolds, when two closed manifolds are glued along a simply connected face and that face is removed, the remaining surface still constitutes a closed 2-manifold (i.e., the boundary of the new polyhedron).
    \item Therefore, $R_{result}$ remains a set of faces describing a closed solid.
\end{enumerate}
\begin{equation}
    \forall R_A, R_B \in \mathbb{S}, \quad (R_A \oplusThreeD R_B) \in \mathbb{S}
\end{equation}
\end{proof}
\begin{theorem}
$R_A \oplusThreeD R_B = R_B \oplusThreeD R_A$.
\end{theorem}

\begin{proof}
Let $R_A$ contain $m$ faces and $R_B$ contain $n$ faces. Based on Eq. \ref{eq:composition_def}:
\begin{equation}
    \label{eq:comm_step1}
    R_A \oplusThreeD R_B = \{ f \mid f \in R_A \cup R_B, f \neq S_A, f \neq S_B \}
\end{equation}
Now consider the reverse operation $R_B \oplusThreeD R_A$:
\begin{equation}
    \label{eq:comm_step2}
    R_B \oplusThreeD R_A = (R_B \setminus \{S_{B}\}) \cup (R_A \setminus \{S_{A}\})
\end{equation}
According to set algebra, the union operation is commutative, i.e., $X \cup Y = Y \cup X$. Therefore:
\begin{equation}
    (R_A \setminus \{S_{A}\}) \cup (R_B \setminus \{S_{B}\}) = (R_B \setminus \{S_{B}\}) \cup (R_A \setminus \{S_{A}\})
\end{equation}
Substituting back into Eq. \ref{eq:comm_step1} and Eq. \ref{eq:comm_step2}, we obtain:
\begin{equation}
    R_A \oplusThreeD R_B = R_B \oplusThreeD R_A
\end{equation}
This indicates that the construction of solid figures is independent of the input order of components.
\end{proof}

\begin{theorem}
$(R_A \oplusThreeD R_B) \oplusThreeD R_C = R_A \oplusThreeD (R_B \oplusThreeD R_C)$.
\end{theorem}

\begin{proof}
setup:
\begin{enumerate}
    \item Polyhedra $A$ and $B$ share interface faces $(S_{AB}, S_{BA})$.
    \item Polyhedra $B$ and $C$ share interface faces $(S_{BC}, S_{CB})$.
    \item $R_A, R_B, R_C$ are their respective face sets.
\end{enumerate}

Left Hand Side:
Let $R_{AB} = R_A \oplusThreeD R_B$.
\begin{equation}
    R_{AB} = (R_A \cup R_B) \setminus \{S_{AB}, S_{BA}\}
\end{equation}
Next, calculate $R_{AB} \oplusThreeD R_C$. The contact interface involves $B$ and $C$ (i.e., $S_{BC}$ and $S_{CB}$):
\begin{align}
    &(R_A \oplusThreeD R_B) \oplusThreeD R_C 
    = (R_{AB} \cup R_C) \setminus \{S_{BC}, S_{CB}\} \label{eq:assoc_left_step1} \\
    &\quad= \left( \left( R_A \cup R_B \right) \setminus \{S_{AB}, S_{BA}\} \cup R_C \right) \setminus \{S_{BC}, S_{CB}\} \label{eq:assoc_left_step2}
\end{align}
Since $S_{AB}, S_{BA}$ exist only between $A$ and $B$ and do not affect $C$, the formula simplifies to the total union minus all interface faces:
\begin{equation}
    \label{eq:assoc_left}
    = (R_A \cup R_B \cup R_C) \setminus \{S_{AB}, S_{BA}, S_{BC}, S_{CB}\}
\end{equation}

Right Hand Side:
Let $R_{BC} = R_B \oplusThreeD R_C$.
\begin{equation}
    R_{BC} = (R_B \cup R_C) \setminus \{S_{BC}, S_{CB}\}
\end{equation}
Next, calculate $R_A \oplusThreeD R_{BC}$. The contact interface involves $A$ and $B$:
\begin{align}
    &R_A \oplusThreeD (R_B \oplusThreeD R_C) 
    = (R_A \cup R_{BC}) \setminus \{S_{AB}, S_{BA}\} \label{eq:assoc_right_step1} \\
    &\quad= \left( R_A \cup \left( (R_B \cup R_C) \setminus \{S_{BC}, S_{CB}\} \right) \right) \setminus \{S_{AB}, S_{BA}\} \label{eq:assoc_right_step2}
\end{align}
Similarly, this simplifies to the total union minus all interface faces:
\begin{equation}
    \label{eq:assoc_right}
    = (R_A \cup R_B \cup R_C) \setminus \{S_{AB}, S_{BA}, S_{BC}, S_{CB}\}
\end{equation}
Comparing Eq. \ref{eq:assoc_left} and Eq. \ref{eq:assoc_right}, the results are identical:
\begin{equation}
    (R_A \oplusThreeD R_B) \oplusThreeD R_C = R_A \oplusThreeD (R_B \oplusThreeD R_C)
\end{equation}
\end{proof}

\subsection{Predicate Library and Theorem Bank}
\label{app:a2}

% 二级附录：用 subsection
\subsubsection{Predicate Library}
\label{subsec:solid_builtin_predicates}
The Solid Predicate Library encompasses 120 predicates, comprising 35 fundamental predicates (Table~\ref{tab:solid_builtin_predicates}) natively integrated into the solver, along with 20 entity types (Table~\ref{tab:solid_entities}), 35 entity relationships (Table~\ref{tab:solid_relations}), and 30 attribute descriptors (Table~\ref{tab:solid_attributions}) defined via a dedicated predicate.

\begin{table}[ht]
\centering
\caption{fundamental predicates}
\label{tab:solid_builtin_predicates}
\small
\begin{tabular}{ccll}
\toprule
\textbf{id} & \textbf{type} & \textbf{name} & \textbf{examples} \\
\midrule
1 & Construction & Coplanar & Coplanar(ABCD) \\
2 & Construction & Cospherical & Cospherical(O,ABC) \\
3 & Construction & Collinear & Collinear(ABC) \\
4 & Construction & Cocircular & Cocircular(O,ABC) \\
5 & BasicEntity & Plane & Plane(U) \\
6 & BasicEntity & Sphere & Sphere(O) \\
 & ......&......&......\\
\bottomrule
\end{tabular}
\end{table}

\begin{table}[ht]
\centering
\caption{Entities defined for solid geometry using the predicate}
\label{tab:solid_entities}
\small
\begin{tabular}{lcc}
\toprule
\textbf{id} & \textbf{type} & \textbf{examples}\\
\midrule
1 & Entity & Cube(ABCDEFGH) \\
2 & Entity & Cuboid(ABCDEFGH) \\
3 & Entity & Cone(O,P) \\
4 & Entity & Cylinder(P,Q)\\
5 & Entity & Prism(P,Q)\\
6 & Entity & TriangularPrism(ABC,DEF)\\
 &......&......\\
\bottomrule
\end{tabular}
\end{table}

\begin{table}[ht]
\centering
\caption{Relations defined for solid geometry using the predicate}
\label{tab:solid_relations}
\small
\begin{tabular}{lcc}
\toprule
\textbf{id} & \textbf{type} & \textbf{examples} \\
\midrule
1 & Relation & ParallelBetweenPlane(U,V) \\
2 & Relation & PerpendicularBetweenPlane(U,V) \\
3 & Relation & PerpendicularBetweenLineAndPlane(AB,U) \\
4 & Relation & IsDiameterOfSphere(AB,O) \\
5 & Relation & IsTangentOfSphere(PA,O) \\
6 & Relation & IsCentreOfSphere(P,O) \\
 & ...... & ...... \\
\bottomrule
\end{tabular}
\end{table}

\begin{table}[ht]
\centering
\caption{Attributions defined for solid geometry using the predicate}
\label{tab:solid_attributions}
\small
\begin{tabular}{lcc}
\toprule
\textbf{id} & \textbf{type} & \textbf{examples}\\
\midrule
1 & Attribution & HeightOfCone(O,P)  \\
2 & Attribution & HeightOfCylinder(P,Q)  \\
3 & Attribution & HeightOfPrism(P,Q)  \\
4 & Attribution & HeightOfTriangularPrism(ABC,DEF)  \\
5 & Attribution & RadiusOfSphere(O)  \\
6 & Attribution & DiameterOfSphere(O)  \\
 &......&......\\
\bottomrule
\end{tabular}
\end{table}

The detailed statements for defining a predicate is as shown in the Tab.~\ref{tab:midpoint_check}, including the predicate name and point variable declaration, validity check declaration, multiple representations, and automatic expansion.
 
\begin{table*}[ht]
\centering
\caption{example for defining a predicate}
\begin{tabular}{lcl}
\hline
name & item & content \\
\hline
\multirow{5}{*}{IsMidpointOfLine(M,AB)} &  & Point(M) \\
 & check & Line(AB) \\
 &  & Collinear(AMB) \\
 & multi & M,BA \\
 & extend & Equal(LengthOfLine(AM),LengthOfLine(MB)) \\
\hline
\end{tabular}
\label{tab:midpoint_check}
\end{table*}

\subsubsection{Theorem Bank}
As shown in Table \ref{tab:theorem_examples}, this section presents some representative theorems from the theorem library, demonstrating the formal representation of geometric knowledge used in the automated reasoning system.

These theorems focus on line-plane relationships and solid geometry: plane-to-plane relationships (transitivity), line-to-plane relationships (perpendicular and parallel judgments), sphere formulas (area and volume), and cylinder properties (volume and height judgments). The theorem library contains over 200 such theorems covering comprehensive geometric knowledge from basic properties to advanced relationships.

\begin{table*}[ht]
\centering
\caption{Examples from the theorem library.}
\label{tab:theorem_examples}
\scriptsize
\setlength{\tabcolsep}{4pt}
\begin{tabularx}{\textwidth}{>{\raggedright\arraybackslash}X>{\raggedright\arraybackslash}X>{\raggedright\arraybackslash}X}
\toprule
\textbf{theorem} & \textbf{premise} & \textbf{conclusion} \\
\midrule
transitivity\_between\_plane\_and\_plane & ParallelBetweenPlane(U,V) \& \allowbreak ParallelBetweenPlane(V,W) & ParallelBetweenPlane(U,W) \\
perpendicular\_judgment\_between\_line\_and\_plane & PerpendicularBetweenLine(AD,BD) \& \allowbreak PerpendicularBetweenLine(AD,CD) \& \allowbreak Coplanar(U,BCD) \& \allowbreak $\sim$Coplanar(U,AD) & PerpendicularBetweenLineAndPlane(AD,U) \\
perpendicular\_judgment\_between\_plane\_and\_plane & PerpendicularBetweenLineAndPlane(AB,U) \& \allowbreak Coplanar(V,AB) & PerpendicularBetweenPlane(U,V) \\
parallel\_judgment\_between\_line\_and\_plane & ParallelBetweenLine(AB,CD) \& \allowbreak Coplanar(U,CD) \& \allowbreak $\sim$Coplanar(U,AB) & ParallelBetweenLineAndPlane(AB,U) \\
sphere\_area\_formula & Sphere(O) & Equal(AreaOfSphere(O),\allowbreak Mul(4,\allowbreak pi,\allowbreak RadiusOfSphere(O),\allowbreak RadiusOfSphere(O))) \\
sphere\_volume\_formula & Sphere(O) & Equal(VolumeOfSphere(O),\allowbreak Mul(4/3,\allowbreak pi,\allowbreak RadiusOfSphere(O),\allowbreak RadiusOfSphere(O),\allowbreak RadiusOfSphere(O))) \\
cylinder\_volume\_formula\_common & Cylinder(P,Q) & Equal(VolumeOfCylinder(P,Q),\allowbreak Mul(AreaOfCircle(P),\allowbreak HeightOfCylinder(P,Q))) \\
height\_of\_cylinder\_judgment & Cylinder(P,Q) \& \allowbreak Coplanar(P,A) \& \allowbreak Coplanar(Q,B) \& \allowbreak PerpendicularBetweenLineAndPlane(AB,P) & Equal(LengthOfLine(AB),\allowbreak HeightOfCylinder(P,Q)) \\
......&......&......\\
\bottomrule
\end{tabularx}
\end{table*}
\clearpage

%% file: sec/aB.tex
\clearpage
\section{Dataset}
\label{app:b}
\subsection{Data and Annotation}
SolidFGeo2k (Fig.~\ref{fig:placeholder11}) presents examples of geometric problems under this benchmark, covering various scenarios and tasks related to solid geometry; PlaneFGeo3k (Fig.~\ref{fig:plane3k}) focuses on the field of plane geometry; MathVerse-solid (Fig.~\ref{fig:mathverse}) is a representative resource among existing geometric benchmarks, providing abundant supplementary data support for solid geometry tasks. Together, the three constitute a multi-scenario and multi-level dataset system for geometric tasks.

\begin{figure}[ht]
    \centering
    \includegraphics[width=1\linewidth]{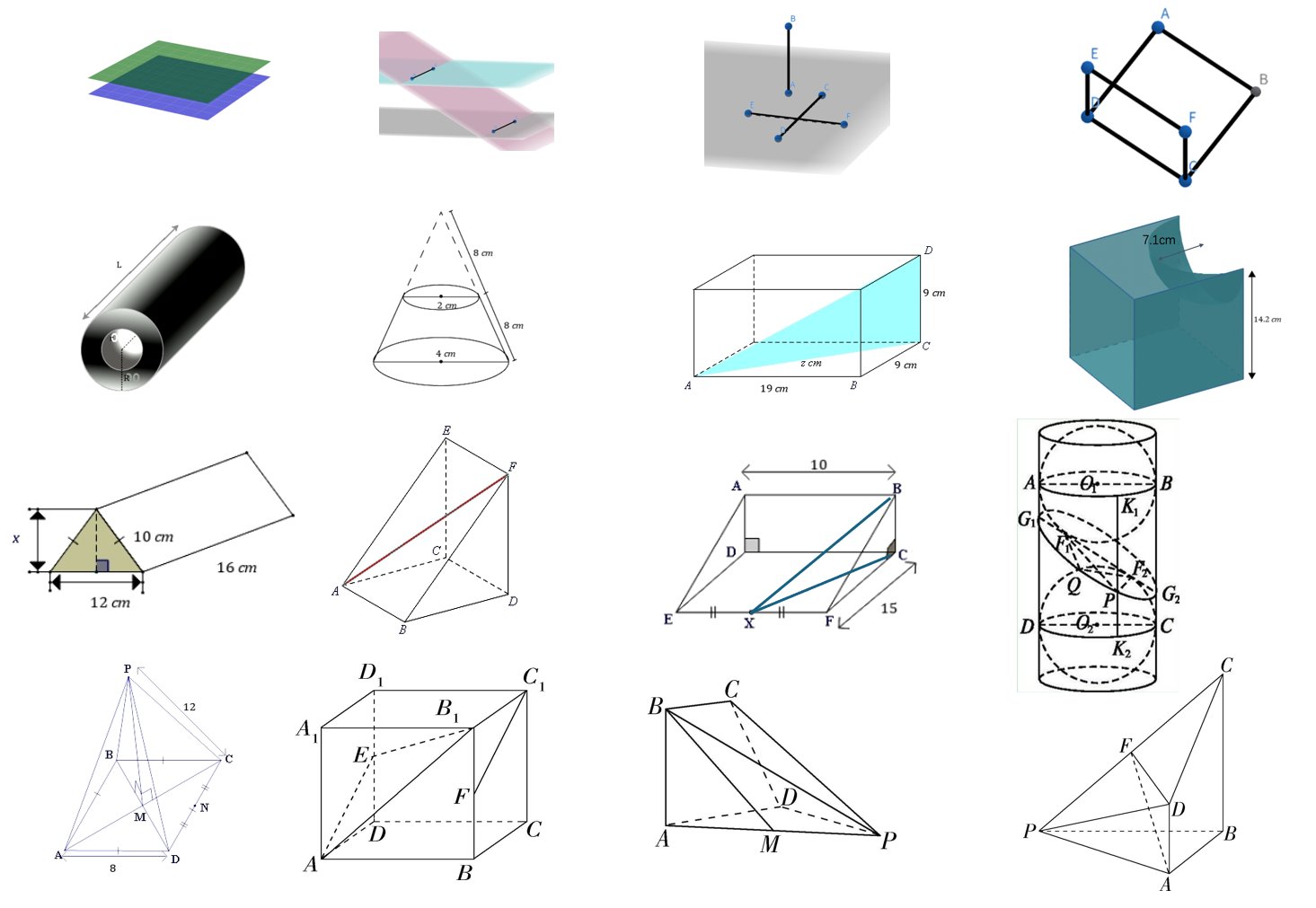}
    \caption{Examples from SolidFGeo2K}
    \label{fig:placeholder11}
\end{figure}
\begin{figure}[ht]
    \centering
    \includegraphics[width=1\linewidth]{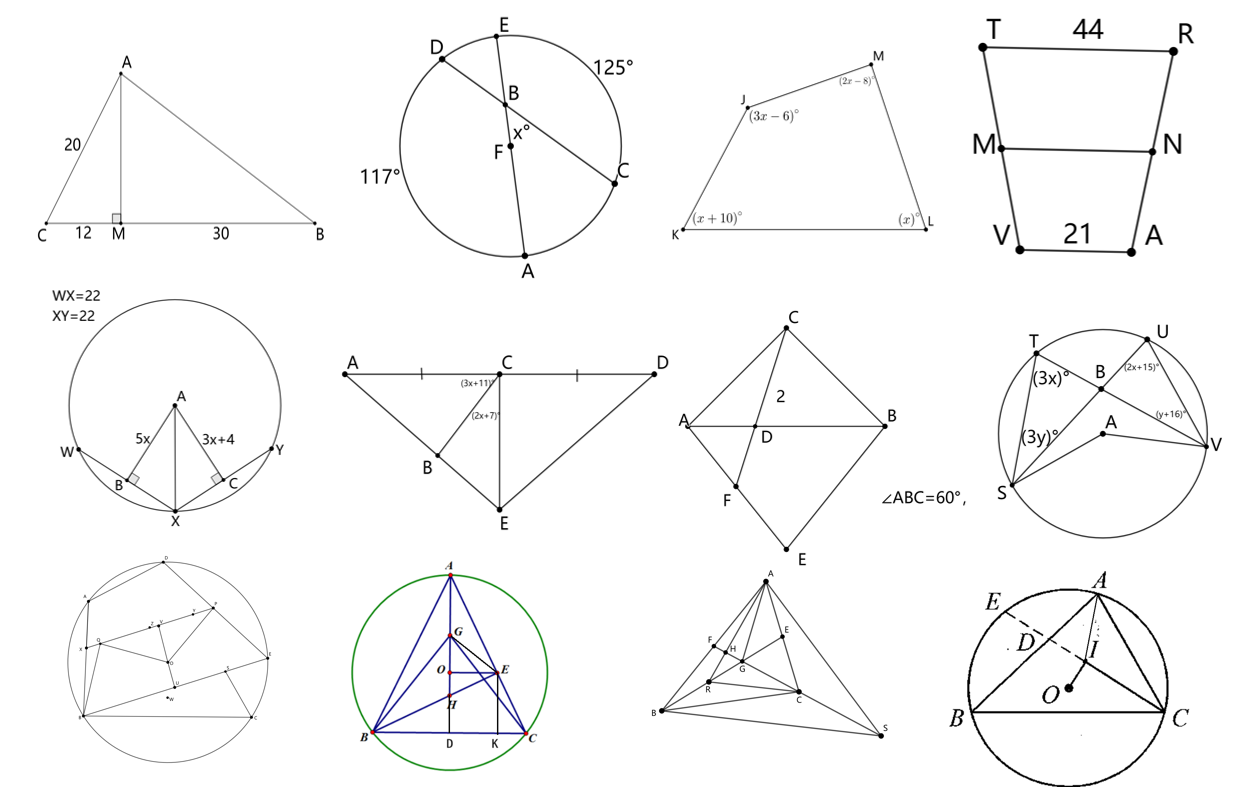}
    \caption{Examples from PlaneFGeo3k}
    \label{fig:plane3k}
\end{figure}
\begin{figure}[ht]
    \centering
    \includegraphics[width=1\linewidth]{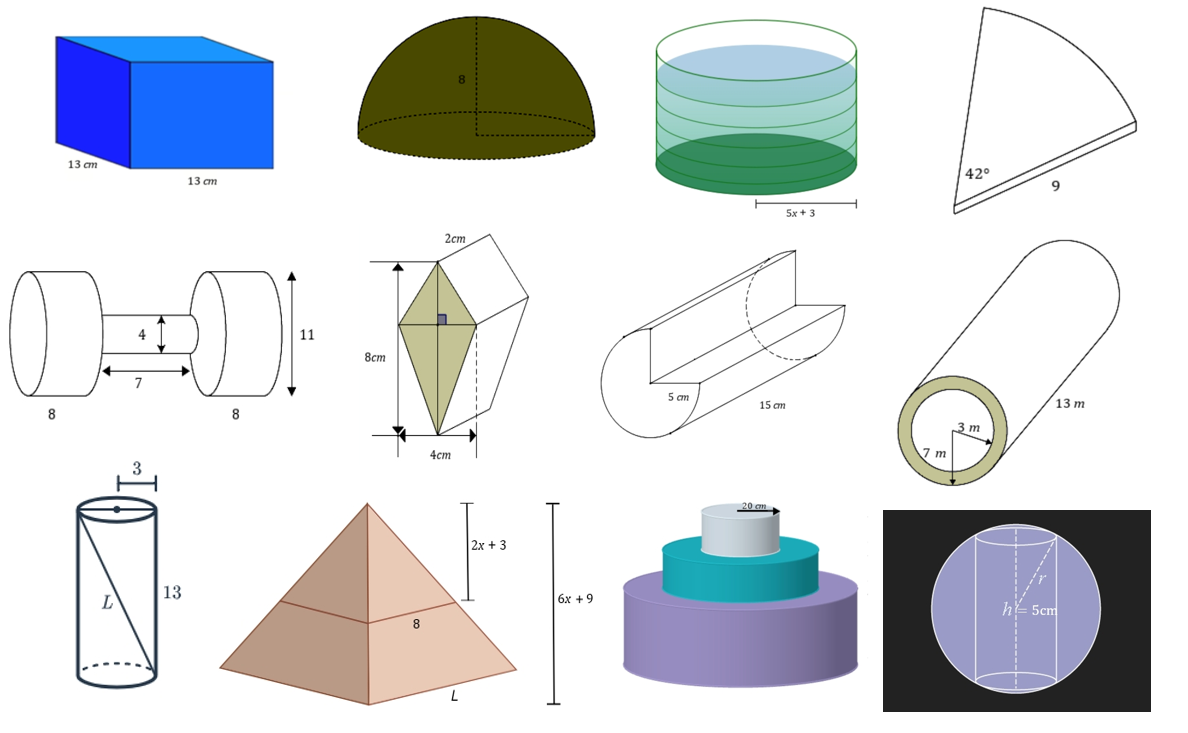}
    \caption{Examples from Mathverse-solid}
    \label{fig:mathverse}
\end{figure}
%专家标注3个数据集，如左边；cdl如图，还有4个细粒度
\begin{figure}
    \centering
    \includegraphics[width=1\linewidth]{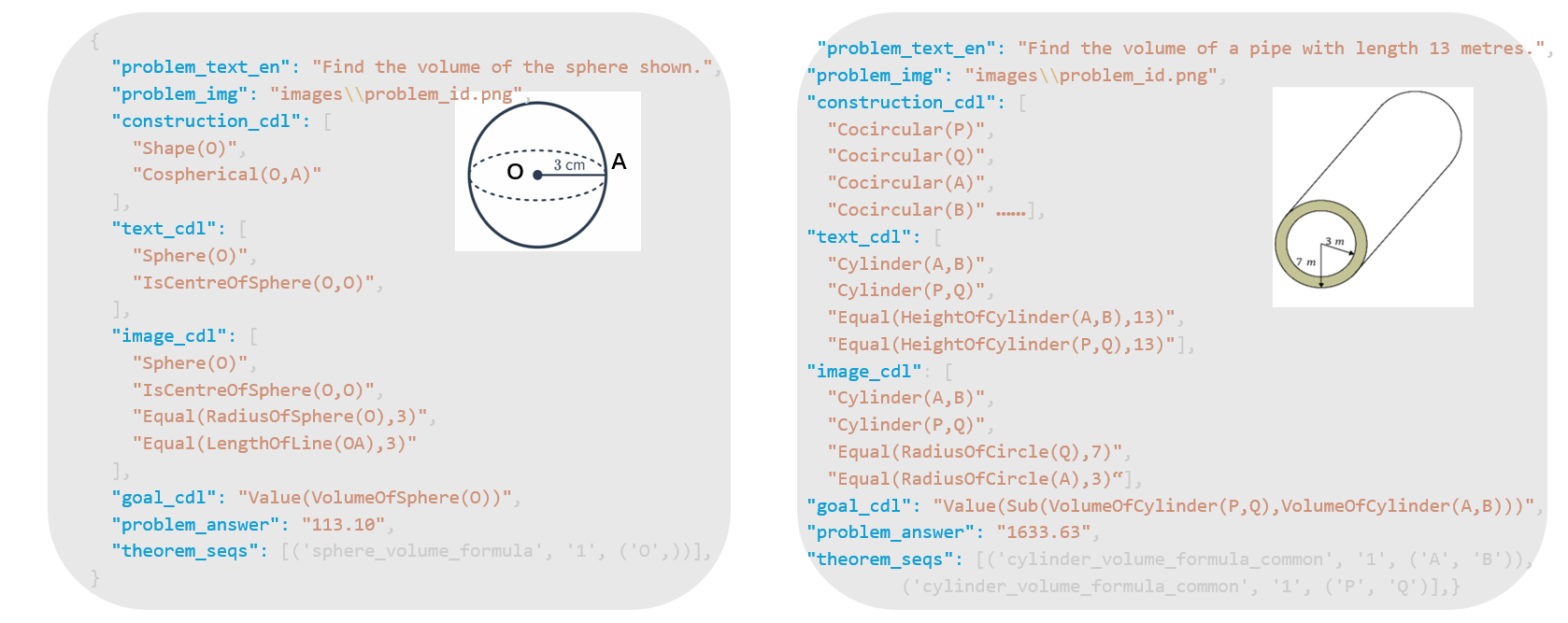}
    \caption{Data Samples from Mathverse-solid and SolidFGeo2k}
    \label{fig:sphere}
\end{figure}

To convert raw problem content into actionable resources for geometric formalization and model training, datasets are systematically annotated manually by experts. This annotation work adheres to standardized protocols and uses professional tools. Meanwhile, formal language encoding is adopted to convert the annotated information into Geometric Description Language. As shown in Figure \ref{fig:sphere}, the geometric description language and the standard answers are annotated based on the original text and images, which serve as the ground truth.

\subsection{Data Construction}
In our dataset SolidFGeo2k, we collect solid geometric problems in existing datasets except SolidGeo (containing other datasets) as shown in Fig~\ref{fig:dataset_comparison}. In addition, we add new samples (31.1\%) for the comprehensive evaluation. For each problem, we ask \textbf{two undergraduat students} to give detailed annotations including parsing results, reasoning process, and correct answer. For your reference, we show differenes to existing datasets in Fig~\ref{fig:dataset_comparison}.   

\begin{figure}[ht]
    \centering
    \includegraphics[width=0.90\linewidth]{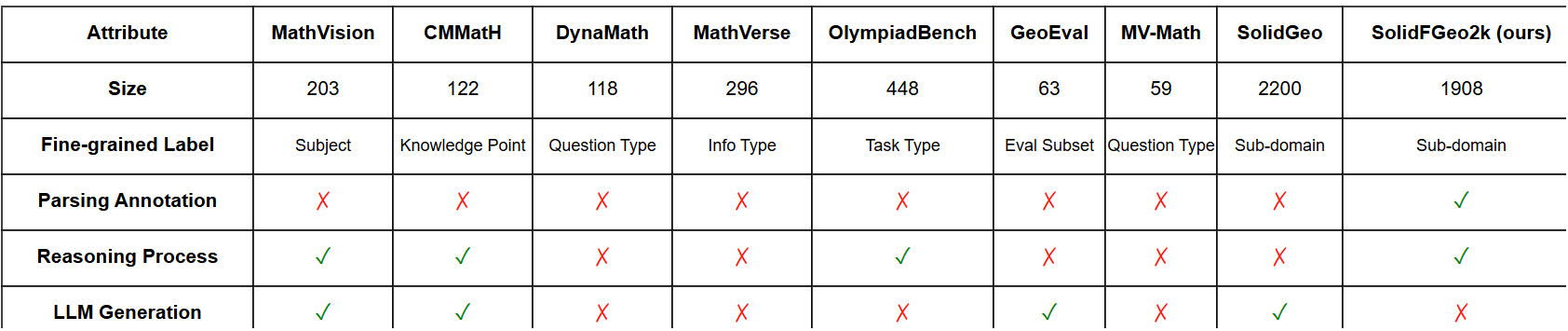}
 
    \caption{Comparison of Datasets for Solid Geometry Problems}
   
    \label{fig:dataset_comparison}
\end{figure}

\subsection{Fine-grained Subject categorization}
\begin{figure}[ht]
    \centering
    \includegraphics[width=1\linewidth]{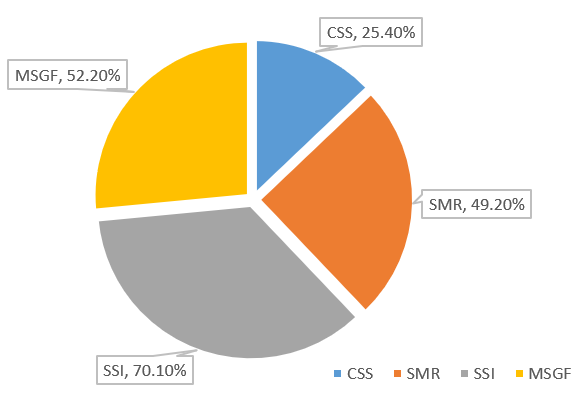}
    \caption{Proportion of different subjects in SolidFGeo2k, note that a single question may match multiple different subjects (Fig. \ref{fig:subject})}
    \label{fig:percent}
\end{figure}
The SolidFGeo2k and Mathverse-solid dataset incorporates fine-grained subject categorization, covering the full spectrum of spatial reasoning skills required. Figure \ref{fig:percent} categorizes SolidFGeo2k’s content into themes: Composite Solid Structure, Spatial Metric Relations, Solid Shape Identification, and Measurement of Solid Geometric Forms, demonstrating the dataset’s coverage of core solid geometry topics. Figure 18’s pie chart quantifies the proportion of different subjects in SolidFGeo2k: SSI (Solid Shape Identification) accounts for 70.10\%, SMR (Spatial Metric Relations) for 49.20\%, CSS (Composite Solid Structure) for 25.40\%, and MSGF (Measurement of Solid Geometric Forms) for 52.20\%. Note that a single problem may belong to multiple subjects, reflecting the interdisciplinary nature of geometric reasoning tasks.
\begin{figure}[ht]
    \centering
    \includegraphics[width=1\linewidth]{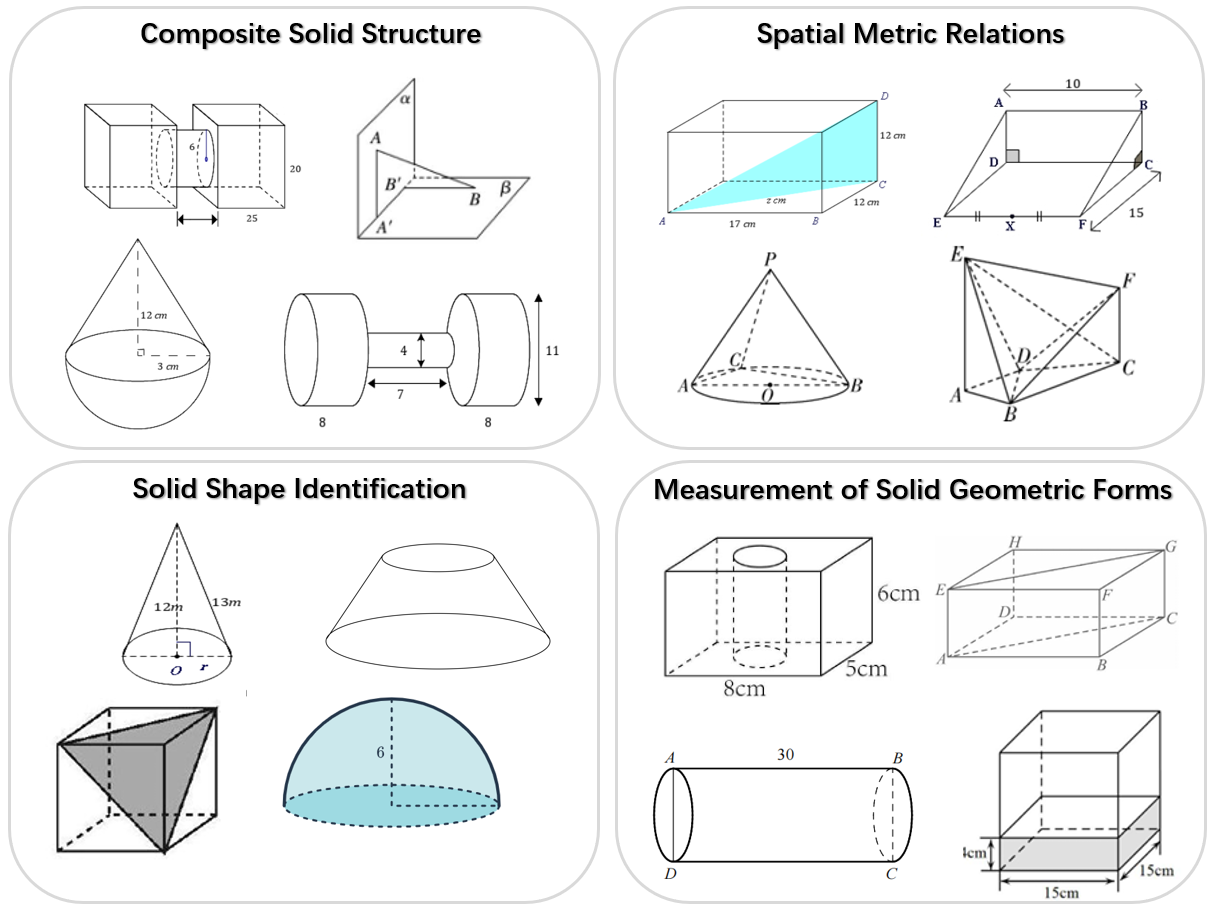}
    \caption{Examples from different subjects}
    \label{fig:subject}
\end{figure}

\clearpage

%% file: sec/aC.tex
\clearpage
\section{Multimodal Formalization Parser (M2FP) Supplementary Information}
\label{app:c}

\subsection{M2FP Pipeline Architecture}
\label{app:c1}

\subsubsection{Overview and Objectives}

The Multimodal Formalization Parser (M2FP) translates natural language geometric problem descriptions and visual diagrams into Conditional Description Language (CDL), enabling precise computational reasoning. It integrates information from text and images, converts descriptions into predicate-based CDL format, and ensures completeness and consistency across modalities.

\subsubsection{Output Structure}
\label{subsubsec:output_structure}

The parser produces a structured CDL representation with four components:

\begin{enumerate}
    \item \textbf{\texttt{construction\_cdl}}: Geometric construction predicates defining fundamental structure (e.g., \texttt{Shape(AB,BC,CD,DA)}, \texttt{Collinear(PABQ)}, \texttt{Cocircular(O)})
    
    \item \textbf{\texttt{text\_cdl}}: Predicates extracted from natural language (e.g., \texttt{Equal(LengthOfLine(A,B),5)}, \texttt{ParallelBetweenLine(A,B,C,D)})
    
    \item \textbf{\texttt{image\_cdl}}: Predicates derived from visual analysis (e.g., \texttt{PerpendicularBetweenLine(AB,BC)}, \texttt{Equal(RadiusOfCircle(O),3)})
    
    \item \textbf{\texttt{goal\_cdl}}: A single canonicalized predicate representing the problem's objective (e.g., \texttt{Value(VolumeOfCone(O,P))})
\end{enumerate}

The system ensures entities in \texttt{text\_cdl} and \texttt{image\_cdl} are declared in \texttt{construction\_cdl}, maintaining cross-modal consistency.

\subsubsection{Pipeline Components}
\label{subsubsec:pipeline_components}

\paragraph{Stage 1: Text Parsing}

The text parsing module extracts entities (points, lines, shapes), maps relations (e.g., ``parallel'' $\rightarrow$ \texttt{ParallelBetweenLine(A,B,C,D)}), parameterizes attributes (e.g., ``length of AB is 5'' $\rightarrow$ \texttt{Equal(LengthOfLine(A,B),5)}), identifies goals (e.g., ``Find the volume'' $\rightarrow$ \texttt{Value(VolumeOfCone(O,P))}), and normalizes the output format.

\paragraph{Stage 2: Image Parsing}

The image parsing module detects primitives (points, lines, shapes), recognizes visual relations (right-angle marks, parallel indicators, intersections), parses numeric annotations from diagram labels, aligns entities with text mentions, resolves conflicts, and encodes visual information into CDL format consistent with text-derived predicates.

\paragraph{Stage 3: Output Packing}

The output packing stage consolidates all information into a JSON structure with four CDL components. Numeric values are formatted without units (supporting expressions like \texttt{36*pi}), entities use uppercase letters, and formatting follows strict rules (no extra spaces). The system validates that all entities are declared in \texttt{construction\_cdl} and all predicates conform to the official vocabulary.

\subsection{Fuzzy Matching Evaluation Metrics}
\label{app:c2}

We employ fuzzy set-based metrics to evaluate CDL parsing outputs, allowing partial credit for semantically similar predicates that may differ in variable naming or formatting. This section provides complete definitions enabling full reproducibility of the evaluation metrics.

\subsubsection{Element Normalization Process}
\label{subsubsec:element_normalization}

Before computing similarity scores, all CDL elements must be normalized to eliminate differences caused by variable naming conventions. Given a raw CDL set $P = \{p_1, p_2, \ldots, p_n\}$ (predicted) or $G = \{g_1, g_2, \ldots, g_m\}$ (ground truth), we apply a normalization function $\text{Normalize}(\cdot)$ to obtain normalized sets $P' = \{\text{Normalize}(p) \mid p \in P\}$ and $G' = \{\text{Normalize}(g) \mid g \in G\}$.

The normalization process follows these steps:
\begin{enumerate}
    \item \textbf{Remove whitespace}: Strip all spaces from the element string
    \item \textbf{Replace variables}: Replace all variable names (single uppercase letters like A, B, C, O, P, etc.) with the placeholder \texttt{\_V\_}
    \item \textbf{Preserve numeric values}: Keep all numeric values unchanged
    \item \textbf{Process nested structures}: Recursively apply normalization to nested predicate structures
    \item \textbf{Normalize variable-only parameters}: If all parameters are variables, replace with \texttt{(\_V\_)}
\end{enumerate}

\textbf{Examples:}
\begin{itemize}
    \item \texttt{Equal(LengthOfLine(A,B),5)} $\rightarrow$ \texttt{Equal(LengthOfLine(\_V\_),5)}
    \item \texttt{Equal(HeightOfCone(O,P),12)} $\rightarrow$ \texttt{Equal(HeightOfCone(\_V\_),12)}
    \item \texttt{Shape(AB,BC,CD,DA)} $\rightarrow$ \texttt{Shape(\_V\_,\_V\_,\_V\_,\_V\_)}
    \item \texttt{Collinear(PABQ)} $\rightarrow$ \texttt{Collinear(\_V\_)}
\end{itemize}

The complete normalization algorithm is provided in Algorithm~\ref{alg:normalize}.

\subsubsection{Fuzzy Matching Score Function}
\label{subsubsec:score_function}

The core fuzzy matching score function $S(p', g') \in [0, 1]$ quantifies the similarity between a normalized predicted element $p' \in P'$ and a normalized ground truth element $g' \in G'$. This function is formally defined as:

\begin{equation}
S(p', g') = \begin{cases}
1.0 & \text{if } p' = g' \text{ (exact match)} \\[0.3em]
0.8 & \text{if } \text{predicate\_name}(p') = \text{predicate\_name}(g') \\
    & \quad \wedge \text{numbers}(p') \cap \text{numbers}(g') \neq \emptyset \\[0.3em]
0.5 & \text{if } \text{predicate\_name}(p') = \text{predicate\_name}(g') \\
    & \quad \text{only} \\[0.3em]
0.3 & \text{if } \text{numbers}(p') \cap \text{numbers}(g') \neq \emptyset \\
    & \quad \text{only} \\[0.3em]
0.0 & \text{otherwise}
\end{cases}
\label{eq:matching_score}
\end{equation}

where:
\begin{itemize}
    \item $\text{predicate\_name}(x)$ extracts the predicate name (the substring before the first parenthesis) from normalized element $x$
    \item $\text{numbers}(x)$ extracts all numeric values from normalized element $x$ using regular expressions
    \item The intersection $\text{numbers}(p') \cap \text{numbers}(g') \neq \emptyset$ means at least one numeric value appears in both elements
\end{itemize}

\textbf{Examples:}
\begin{itemize}
    \item $S(\texttt{Equal(\_V\_,5)}, \texttt{Equal(\_V\_,5)}) = 1.0$ (exact match)
    \item $S(\texttt{Equal(\_V\_,5)}, \texttt{Equal(\_V\_,5.0)}) = 0.8$ (predicate match + number intersection)
    \item $S(\texttt{Equal(\_V\_,5)}, \texttt{Equal(\_V\_,3)}) = 0.5$ (predicate match only)
    \item $S(\texttt{Equal(\_V\_,5)}, \texttt{LengthOf(\_V\_,5)}) = 0.3$ (number intersection only)
    \item $S(\texttt{Equal(\_V\_,5)}, \texttt{LengthOf(\_V\_,3)}) = 0.0$ (no match)
\end{itemize}

\subsubsection{Fuzzy Jaccard Similarity}
\label{subsubsec:fuzzy_jaccard}

Given normalized sets $P'$ and $G'$ (obtained from raw sets $P$ and $G$ via the normalization process in Section~\ref{subsubsec:element_normalization}), and using the scoring function $S(p', g')$ defined in Section~\ref{subsubsec:score_function}, the Fuzzy Jaccard Similarity extends the standard Jaccard index by replacing exact set intersection with a fuzzy intersection. The fuzzy Jaccard score is computed as:

\begin{equation}
\text{Score}(P, G) = \frac{\sum_{p' \in P'} \max_{g' \in G'} S(p', g')}{|P'| + |G'| - \sum_{p' \in P'} \max_{g' \in G'} S(p', g')}
\label{eq:fuzzy_jaccard}
\end{equation}

where:
\begin{itemize}
    \item The numerator $\sum_{p' \in P'} \max_{g' \in G'} S(p', g')$ represents the \textbf{fuzzy intersection}: the sum of best-match scores for each predicted element
    \item The denominator $|P'| + |G'| - \sum_{p' \in P'} \max_{g' \in G'} S(p', g')$ represents the \textbf{fuzzy union}: the sum of set sizes minus the intersection
    \item $P'$ and $G'$ are obtained by applying normalization to raw sets $P$ and $G$ as described in Section~\ref{subsubsec:element_normalization}
\end{itemize}

The complete computation algorithm is provided in Algorithm~\ref{alg:fuzzy_metrics}.

\subsubsection{Boundary Case Handling}
\label{subsubsec:boundary_cases}

Boundary cases are handled as follows: if both sets are empty, Jaccard = 1.0; if one set is empty, Jaccard = 0.0; division by zero is prevented with explicit checks.

\subsubsection{Evaluation Algorithms}

The evaluation framework implements three core algorithms: normalization (Algorithm~\ref{alg:normalize}), Fuzzy Jaccard calculation (Algorithm~\ref{alg:fuzzy_metrics}), and score matrix construction (Algorithm~\ref{alg:score_matrix}).

\begin{algorithm}[t]
\caption{CDL Element Normalization}
\label{alg:normalize}
\begin{algorithmic}[1]
\Function{NormalizeCDLElement}{element}
    \State $t \gets$ strip spaces(element)
    \If{no parentheses in $t$}
        \If{is number($t$)}
            \State \Return $t$
        \Else
            \State \Return replace all variables with \texttt{\_V\_}
        \EndIf
    \Else
        \State Process innermost parentheses recursively
        \State $parts \gets$ split by commas within parentheses
        \For{each $part$ in $parts$}
            \If{is number($part$)}
                \State Keep numerical value
            \ElsIf{contains parentheses}
                \State $part \gets$ \Call{NormalizeCDLElement}{$part$}
            \Else
                \State $part \gets$ \texttt{\_V\_}
            \EndIf
        \EndFor
        \If{all parts are \texttt{\_V\_}}
            \State \Return \texttt{(\_V\_)}
        \Else
            \State \Return concatenate processed parts
        \EndIf
    \EndIf
\EndFunction
\end{algorithmic}
\end{algorithm}

\begin{algorithm}[t]
\caption{Fuzzy Jaccard Similarity Calculation}
\label{alg:fuzzy_metrics}
\begin{algorithmic}[1]
\Function{FuzzyJaccard}{$P$, $G$}
    \State $P' \gets \{\Call{NormalizeCDLElement}{p} \mid p \in P\}$
    \State $G' \gets \{\Call{NormalizeCDLElement}{g} \mid g \in G\}$
    \If{boundary case detected}
        \State \Return boundary Jaccard value (as described in Section~\ref{subsubsec:boundary_cases})
    \EndIf
    \State $S \gets$ \Call{BuildScoreMatrix}{$P'$, $G'$}
    \State $row\_max[p'] \gets \max_{g' \in G'} S[p', g']$ for each $p' \in P'$
    \State $intersection \gets \sum_{p' \in P'} row\_max[p']$
    \State $union \gets |P'| + |G'| - intersection$
    \State $jaccard \gets intersection / union$
    \State \Return $jaccard$
\EndFunction
\end{algorithmic}
\end{algorithm}

\begin{algorithm}[t]
\caption{Matching Score Matrix Construction}
\label{alg:score_matrix}
\begin{algorithmic}[1]
\Function{BuildScoreMatrix}{$P'$, $G'$}
    \State Initialize $S$ as $|P'| \times |G'|$ matrix
    \For{each $p' \in P'$ and each $g' \in G'$}
        \If{$p' = g'$}
            \State $S[p', g'] \gets 1.0$
        \Else
            \State $pred\_name \gets$ predicate name($p'$)
            \State $gt\_name \gets$ predicate name($g'$)
            \State $pred\_nums \gets$ extract numbers($p'$)
            \State $gt\_nums \gets$ extract numbers($g'$)
            \State $name\_equal \gets (pred\_name = gt\_name)$
            \State $num\_intersect \gets$ intersects($pred\_nums$, $gt\_nums$)
            \If{$name\_equal$ and $num\_intersect$}
                \State $S[p', g'] \gets 0.8$
            \ElsIf{$name\_equal$}
                \State $S[p', g'] \gets 0.5$
            \ElsIf{$num\_intersect$}
                \State $S[p', g'] \gets 0.3$
            \Else
                \State $S[p', g'] \gets 0.0$
            \EndIf
        \EndIf
    \EndFor
    \State \Return $S$
\EndFunction
\end{algorithmic}
\end{algorithm}

\subsubsection{Computational Efficiency}

The algorithm computes a $|P'| \times |G'|$ score matrix with $O(|P'| \cdot |G'|)$ time complexity, providing a balance between matching quality and efficiency.

\subsection{Parsing Performance Across Models and Sample Sizes}
\label{app:c3}

This section presents comprehensive experimental results evaluating the parsing performance of various Multimodal Large Language Models (MLLMs) under different few-shot sample configurations. The evaluation employs fuzzy matching metrics as described in Section~\ref{subsec:fuzzy_metrics} to assess the quality of CDL parsing across construction predicates, condition predicates (merged text and image), and goal identification.

\subsubsection{Parse Score: Comprehensive Performance Metric}

The \textbf{Parse Score} serves as a comprehensive performance metric that aggregates the parsing quality across three critical components of geometric problem formalization: \textbf{Construction}, \textbf{Condition}, and \textbf{Goal}. The Parse Score is computed as the expectation (average) of the Jaccard similarity scores across these three components:

\begin{equation}
\text{Parse Score} = \frac{1}{3}\sum_{c \in \mathcal{I}} \text{Jaccard}_c
\end{equation}

where $\text{Jaccard}_c$ denotes the Jaccard similarity score for component $c$, $\mathcal{I} = \{\text{Construction}, \text{Condition}, \text{Goal}\}$

\subsubsection{Component-wise Parsing Performance}

To provide detailed insights into the parsing performance, we examine each component individually. The following figures present the Jaccard similarity scores for Construction, Condition, and Goal parsing, respectively, allowing for granular analysis of model strengths and weaknesses across different aspects of geometric formalization.

\begin{figure}[ht]
\centering
\includegraphics[width=0.5\textwidth]{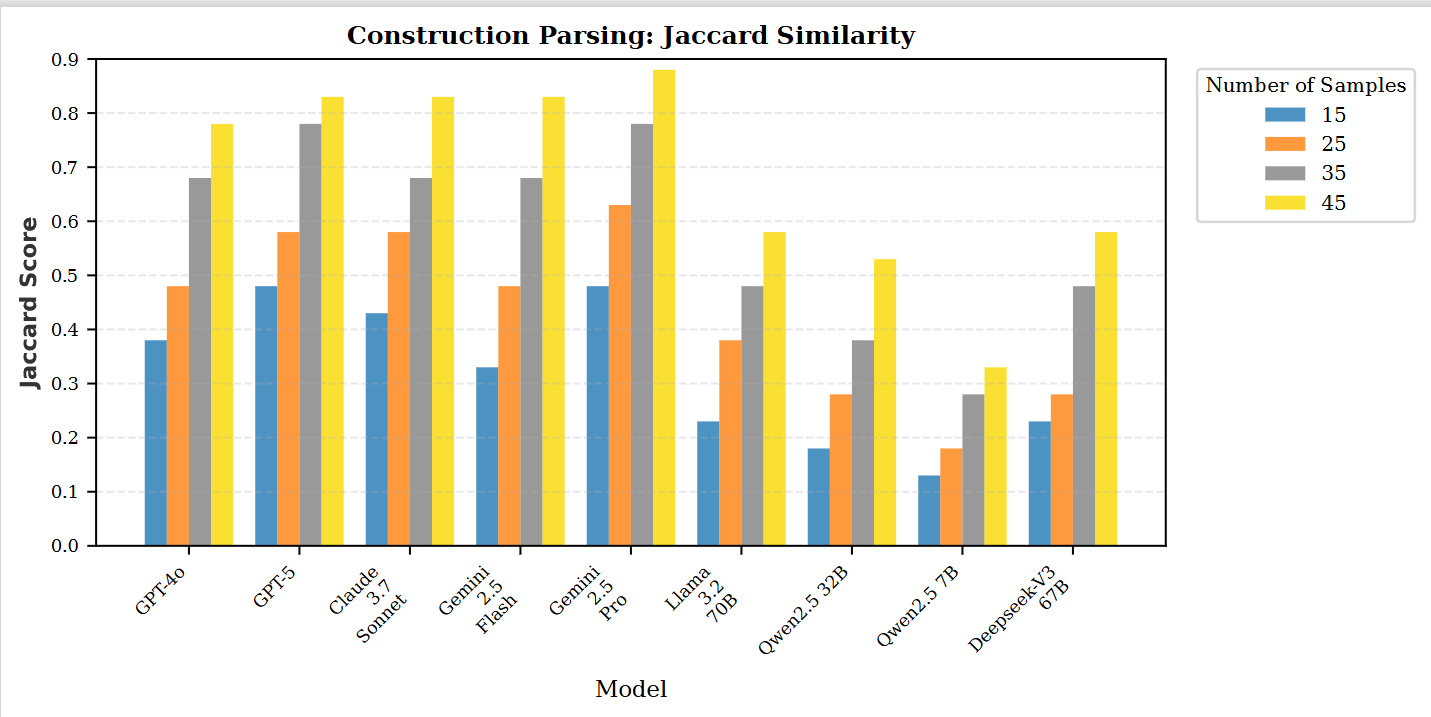}
\caption{Construction Parsing Performance: Jaccard similarity scores across different models and sample sizes. Construction parsing evaluates the accuracy of extracting geometric structure predicates (e.g., \texttt{Shape}, \texttt{Collinear}, \texttt{Cocircular}). All metrics use fuzzy matching as described in Section~\ref{subsec:fuzzy_metrics}.}
\label{fig:construction_jaccard}
\end{figure}

\begin{figure}[ht]
\centering
\includegraphics[width=0.47\textwidth]{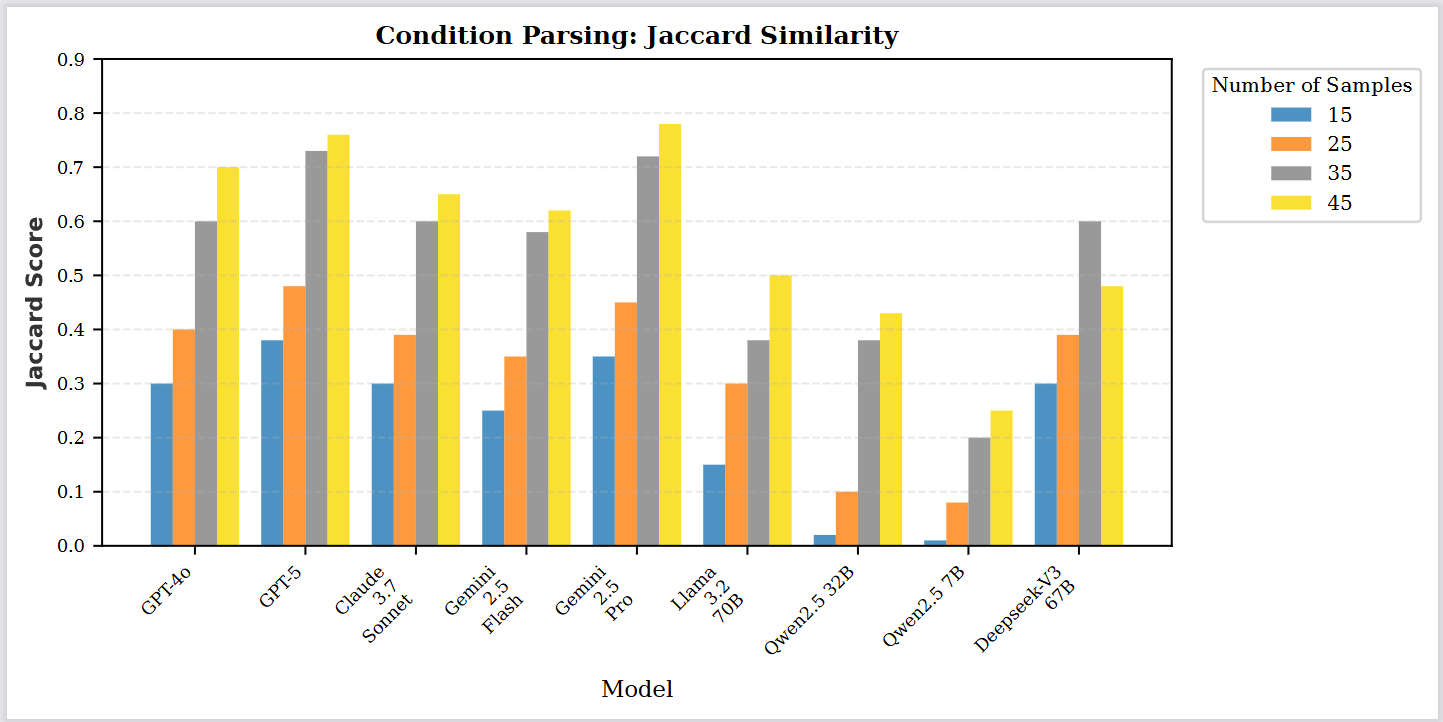}
\caption{Condition Parsing Performance: Jaccard similarity scores across different models and sample sizes. Condition parsing measures the quality of extracting geometric constraints from problem descriptions and visual diagrams (merged text and image). All metrics use fuzzy matching as described in Section~\ref{subsec:fuzzy_metrics}.}
\label{fig:condition_jaccard}
\end{figure}

\begin{figure}[ht]
\centering
\includegraphics[width=0.47\textwidth]{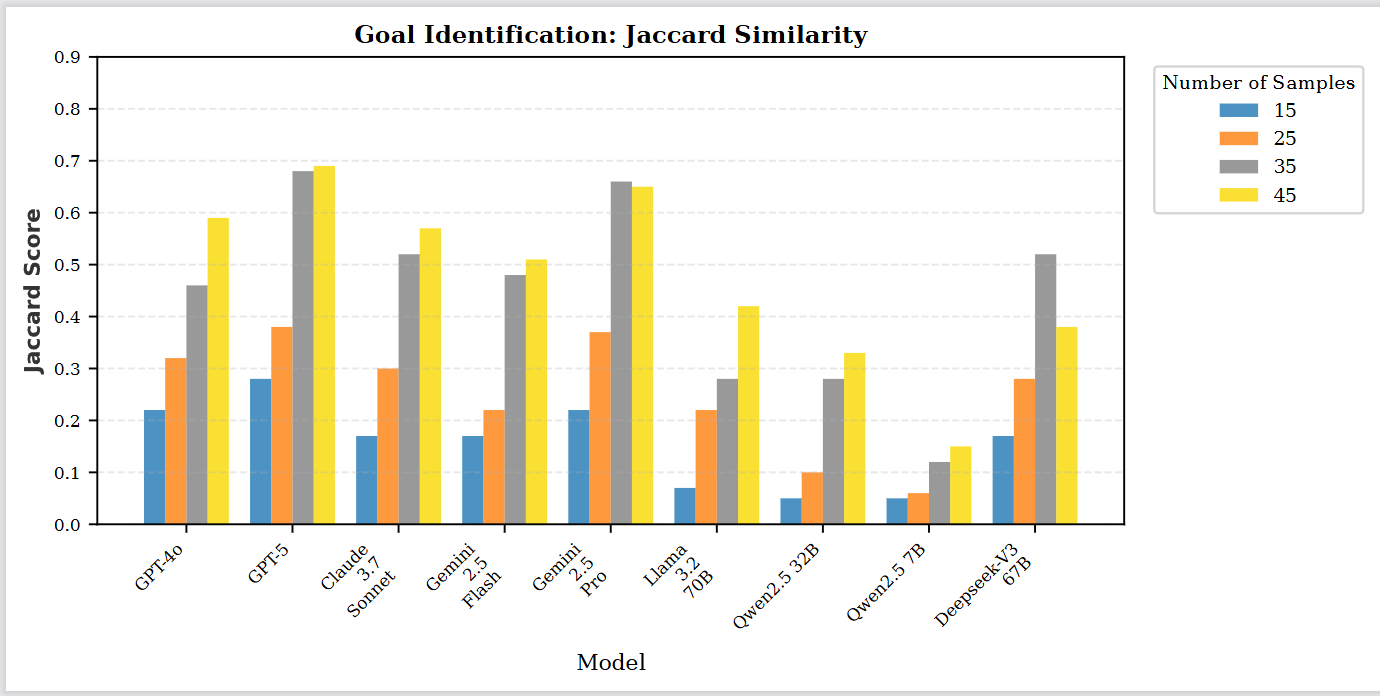}
\caption{Goal Parsing Performance: Jaccard similarity scores across different models and sample sizes. Goal parsing assesses the precision of identifying the problem's objective. All metrics use fuzzy matching as described in Section~\ref{subsec:fuzzy_metrics}.}
\label{fig:goal_jaccard}
\end{figure}

These three components collectively contribute to the Parse Score, which is computed as their average.

\subsubsection{Performance Analysis}

The experimental results (Figures~\ref{fig:construction_jaccard}, \ref{fig:condition_jaccard}, and \ref{fig:goal_jaccard}) show that Parse Score generally improves as sample size increases from 15 to 45 across all models. Construction parsing achieves the highest scores due to the structured nature of construction predicates, while goal parsing remains the most challenging task. Closed-source models consistently outperform open-source models, with larger models showing better performance within the same model family. The Parse Score effectively captures overall parsing quality by equally weighting the three components, ensuring comprehensive parsing capability is rewarded.

\subsection{Prompt Templates}
\label{app:c4}

This section introduces the specific prompt templates used in the M2FP system for CDL parsing and direct problem solving. The M2FP system uses prompt engineering to guide large language models: (1) task description with role definition and schema compliance, (2) strict predicate vocabulary constraints from GDL, (3) few-shot examples covering all predicate types, (4) rule-based guidance for source separation and formatting, and (5) error prevention warnings. This approach balances flexibility with strict constraints to ensure output quality.

\subsubsection{CDL Parsing Prompt Template}
\label{subsubsec:prompt_rules}

The following prompt template is used for \textbf{CDL parsing}, which guides large language models to convert natural language geometric problem descriptions and visual diagrams into formal CDL representations. Parsing step starts with a targeted prompt engineering approach: specialized prompts are constructed using expertly designed example sets that cover all basic predicates. These examples can express predicates and formal language without involving complex geometric relationships, and each includes geometric descriptions and their corresponding formalizations.

\begin{figure}[ht]
    \centering
    \includegraphics[width=1\linewidth]{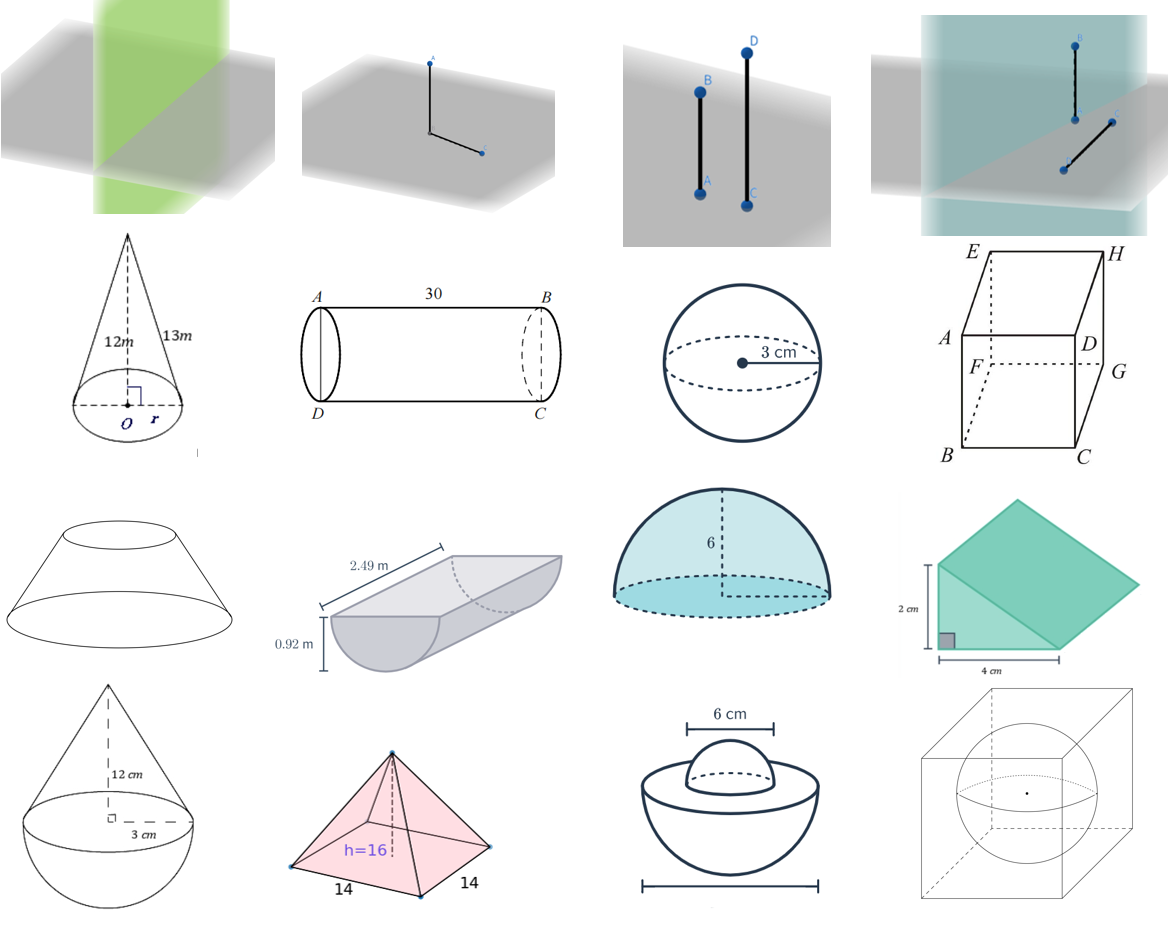}
    \caption{Examples from designed samples}
    \label{fig:45samples}
\end{figure}

The complete prompt template is shown below:

\begin{verbatim}
You are an expert in geometry, logic, and 
computer science. Your task is to precisely 
convert a geometry problem (with natural 
language and an image) into a JSON object 
following the provided JSON Schema. You 
must strictly follow the schema and output 
a complete JSON object.

Rule 0: Predicate Compliance 
(MOST IMPORTANT)
- All CDL predicates you generate 
  (e.g., Equal, Cone, LengthOfLine) MUST 
  be strictly chosen from the official 
  list below.
- Using any predicate that does not 
  appear in this list is strictly 
  forbidden.

--- Official Predicate List ---
[valid_predicates_str]
--- End of Official Predicate List ---

Core Rules and Constraints:

1) Information Source Separation:
   - text_cdl MUST include only facts 
     extracted from the natural language 
     description.
   - image_cdl MUST include only facts 
     directly observable from the image 
     (e.g., length labels, right-angle 
     marks, shape recognition).
   - If a fact appears in both text and 
     image, include it in both fields.

2) construction_cdl - Geometric 
   construction predicates (IMPORTANT):
   construction_cdl defines basic 
   construction for entities, and MUST 
   include the following types where 
   applicable:
   - Shape predicates: define 
     edges/segments of shapes
     * For segments/edges: 
       Shape(AB,BC,CD,DA) or Shape(OP,PO) 
       or Shape(PQ,QP)
     * For points (spheres etc.): 
       Shape(O) or Shape(P)
     * Example: rectangles require 
       Shape(AB,BC,CD,DA); cylinders 
       require Shape(PQ,QP)
   - Collinearity/Cocircular/Coplanar/
     Cospherical:
     * Collinear(PABQ) - P, A, B, Q are 
       collinear
     * Cocircular(O) - O is on a circle 
       (for cone/cylinder base center)
     * Coplanar(U,ABCD) - U coplanar with 
       ABCD
     * Cospherical(O) - O is on a sphere 
       (for spheres)
   Important:
   - Carefully analyze the image to 
     identify all necessary 
     edges/segments/relations
   - Most problems require at least one 
     Shape(...)
   - Cones/cylinders often need Shape(...) 
     and Cocircular(...)
   - Spheres often need Shape(O) and 
     Cospherical(O)

3) Answer formatting:
   - problem_answer MUST be a pure number 
     or expression (e.g., "10", "36*pi"), 
     and MUST NOT contain units or extra 
     text.

4) Core predicate logic:
   - Length/Height: 
     Equal(LengthOfLine(A,B),5), 
     Equal(HeightOfCone(O,P),12)
   - Relations: 
     PerpendicularBetweenLine(A,B,C,D), 
     ParallelBetweenLine(A,B,C,D)
   - Goal: the requested quantity MUST be 
     wrapped by Value(...).

5) Predicate and Operator Legality 
   (CRITICAL):
   - Only reuse names from the official 
     predicate list; DO NOT invent new 
     construction predicates.
   - Quantities allowed in CDL expressions 
     are LIMITED to standard forms: 
     VolumeOfCone, 
     SurfaceAreaOfCylinder, 
     AreaOfCircle, LengthOfLine, etc.
   - Only the following algebraic 
     operators are allowed: Value, Add, 
     Sub, Mul, Div.
   - Formatting: NO extra spaces inside 
     any predicate/operator.

6) Completeness Checks:
   - Ensure every entity used by 
     text_cdl/image_cdl exists in 
     construction_cdl
   - Ensure the target entity in goal_cdl 
     exists in the construction as well
   - Self-check after generation: verify 
     all predicates/operators are allowed, 
     no extra spaces, and no undeclared 
     entities are referenced.

Important: Output Requirements
1. You MUST output a complete JSON object 
   with all required fields
2. All CDL fields MUST be arrays of 
   strings
3. goal_cdl MUST be a string (e.g., 
   "Value(VolumeOfCone(O,P))")
\end{verbatim}

\subsubsection{Direct Problem Solving Prompt}
\label{subsubsec:direct_solving_prompt}

In addition to CDL generation, the system also supports \textbf{direct problem solving} using GPT models for \textbf{testing model accuracy}. This approach bypasses formalization and directly generates answers to geometry problems, providing a baseline for comparison with formalized CDL-based reasoning systems.

The direct solving prompt is simpler than CDL generation and focuses on step-by-step reasoning. The prompt template is:

\begin{verbatim}
You are an expert in geometry and 
mathematics. Please solve the following 
geometry problem step by step.

Important Instructions:
1. Carefully analyze the problem text and 
the accompanying image
2. Show your reasoning process step by step
3. At the end, provide your final answer in 
a clear format
4. **Your final answer should be ONLY a 
number or mathematical expression (like 
"10", "5.5", "12*pi", "36*pi"), without any 
units or text**
5. Put your final answer on a line starting 
with "FINAL ANSWER: "

Example format:
FINAL ANSWER: 10

or 

FINAL ANSWER: 36*pi

Now, please solve this problem:
\end{verbatim}

%% file: sec/aD.tex
\clearpage
\section{SGRE Supplementary Information}
\label{app:d}
\subsection{Theorem Search Tree and Search Process}
\label{app:d1}
\begin{figure}[ht]
    \centering
    \includegraphics[width=1\linewidth]{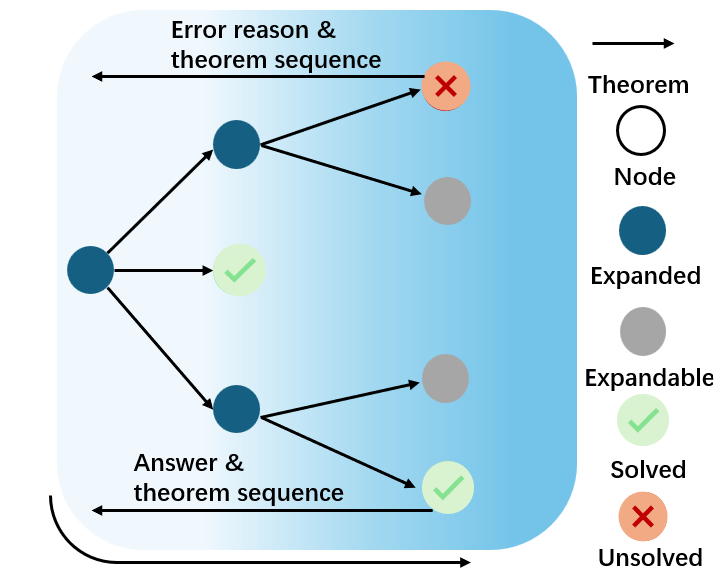}
    \caption{Theorem Search Tree and a inference demonstration is shown in fig.~\ref{fig:bigtree}}
    \label{fig:placeholder13}
\end{figure}

The search process involves constructing a search tree, where nodes represent a set of known conditions, this constitutes a state variable, while arrows denote the applied theorems that supplement the existing 0 condition. As shown in figure \ref{fig:placeholder13}, starting from the known conditions of the problem, theorems are continuously applied to derive new conditions until the goal is reached. The search tree is traversed using either breadth-first or depth-first search, returning nodes in the "expandable" state. The theorems associated with the current nodes are then repeatedly applied to verify if the problem has been solved. Guided by the known conditions of the current node, the "expand" operation checks the list of applicable theorems and generates new nodes. A simple example of "expand" is as follows: given a cube with a known side length, it is intuitive to expand to other side lengths and various angles. This process requires no additional theorem application and can also validate contradictory conditions resulting from upstream parsing.

\subsection{Traditional Search vs. SGRE}
\label{app:d2}
\begin{figure}[ht]
    \centering
    \includegraphics[width=1\linewidth]{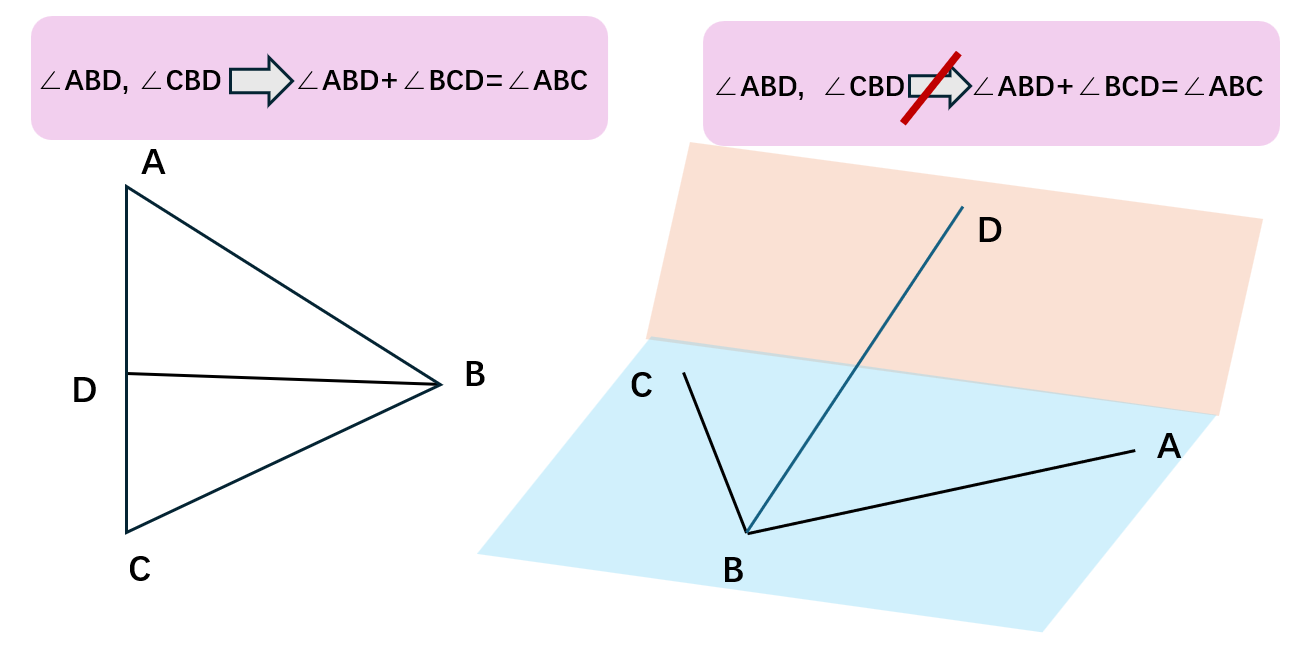}
    \caption{Plane and Solid differ greatly in all aspects, even if the CDL is consistent. SGRE needs to distinguish the Combination of multiple theories}
    \label{fig:placeholder15}
\end{figure}
\begin{figure}[ht]
    \centering
    \includegraphics[width=1\linewidth]{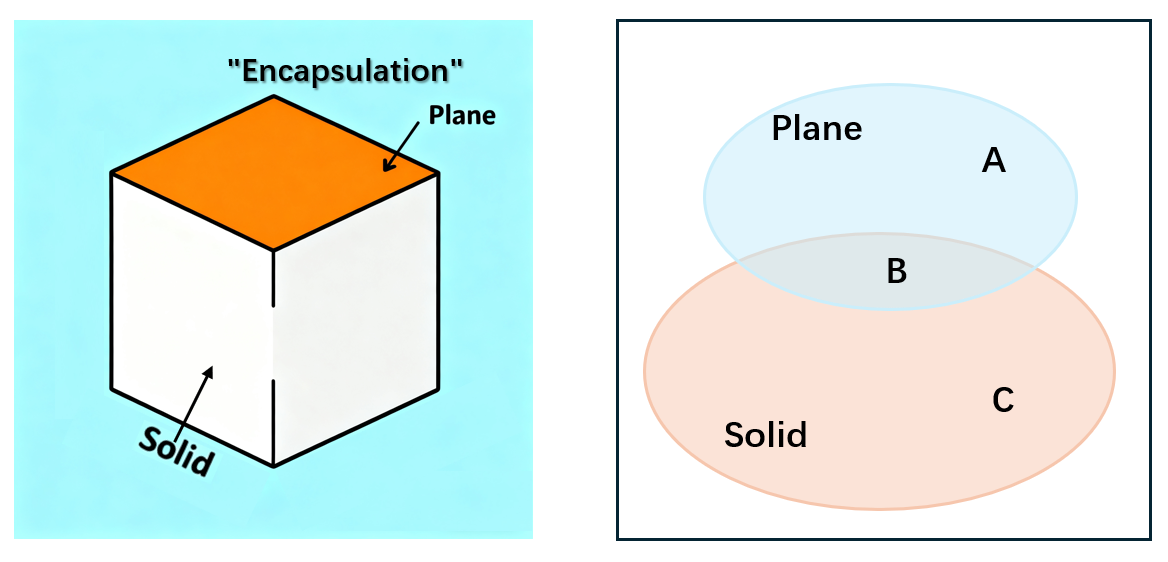}
    \caption{Regarding the encapsulation of the planar theorem, Coplanar(point\_seqs) will be treated as a Plane class to apply the theorems and "extend" in set A (See Fig. \ref{fig:placeholder16})}
    \label{fig:placeholder14}
\end{figure}
\begin{figure}[ht]
    \centering
    \includegraphics[width=1\linewidth]{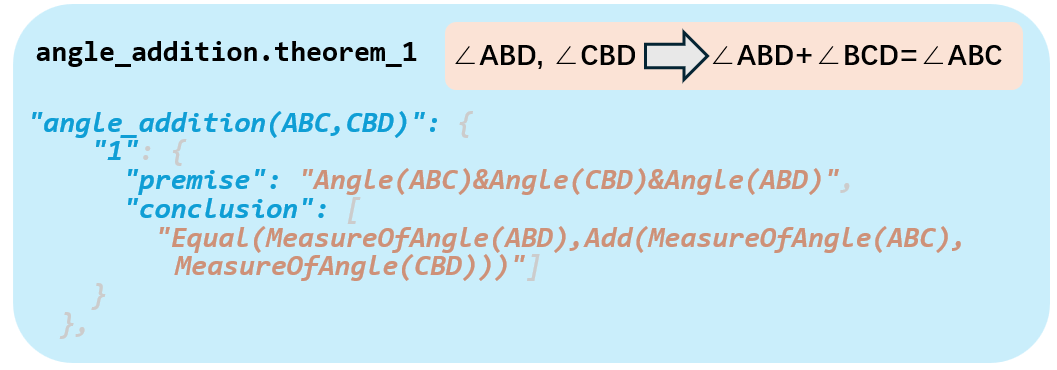}
    \caption{A theorem (as illustrated in Fig. \ref{fig:placeholder15}) that falls into set A holds true exclusively within the Plane Class}
    \label{fig:placeholder16}
\end{figure}
Traditional search has various optimization schemes, such as heuristics, but simple pruning of search is not the focus of this paper. The traditional search focuses only on "inferring the target from the known". In contrast, SGRE draws inspiration from SMT (Satisfiability Modulo Theories), which not only determines satisfiability and proactively identifies contradictions but also addresses multi-constraint conflicts (e.g., given a cube with two known distinct side lengths, it returns "unsatisfiable" and provides the cause of the conflict—specifically, the contradictory constraints) and combinations of multiple theories (e.g., hierarchical rapid evaluation of planar and solid properties: the constraints required for problem-solving vary under different theoretical frameworks; see Figure \ref{fig:placeholder15}). This demonstrates the feasibility of subsequent integration between Hilbert-Geo and large-scale mathematical tools. %such as Lean4.
%Notably, existing large-scale mathematical tools (e.g., Lean) lack the capacity for formalized solution of geometric problems.

\begin{figure*}
    \centering
    \includegraphics[width=1\linewidth]{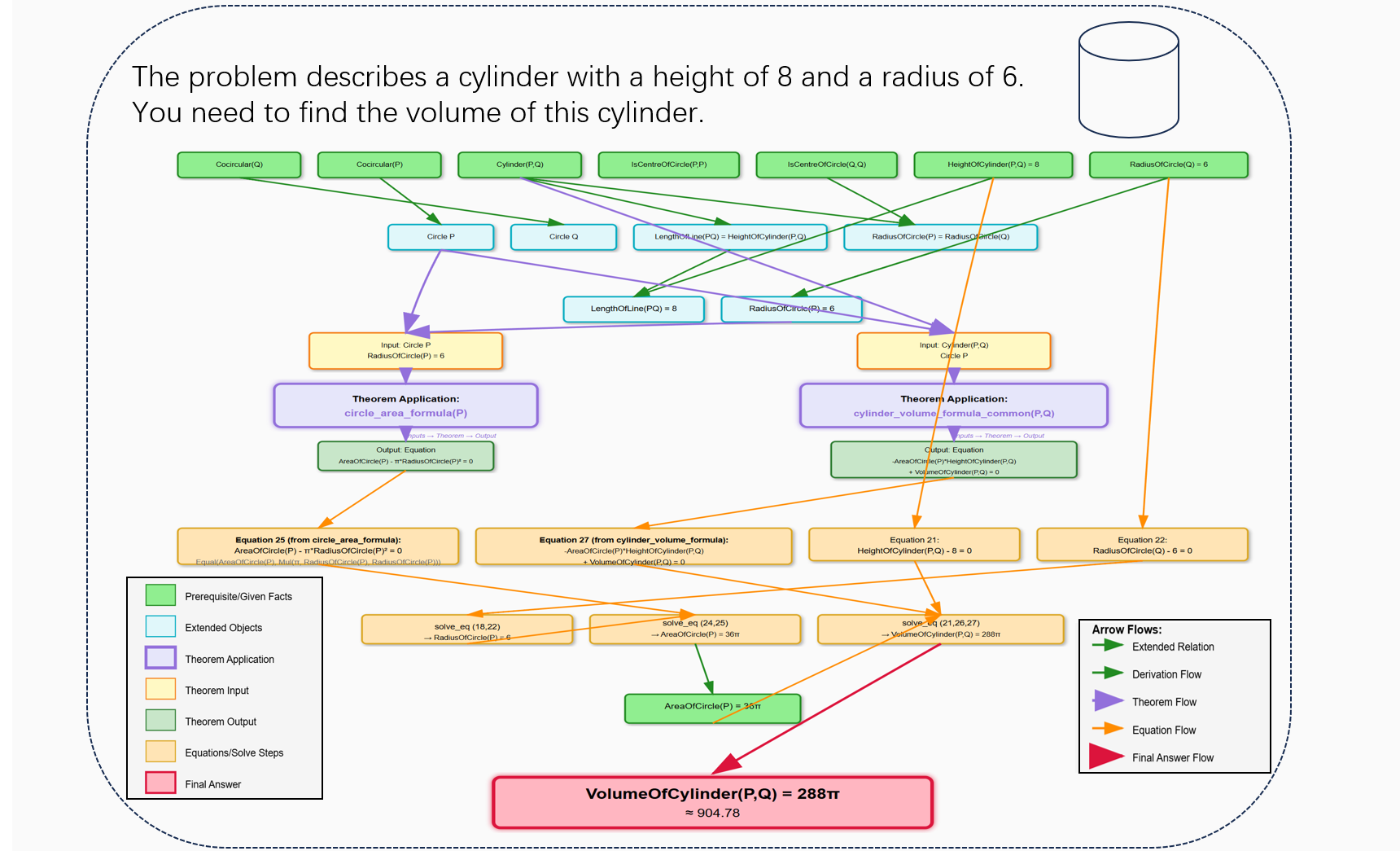}
    \caption{one illustrative example for reasoning based on Hilbert-Geo}
    \label{fig:bigtree}
\end{figure*}
\clearpage

%% file: sec/aE.tex
\clearpage
\section{Experiments}
\label{app:e}
This section supplements some content to the experimental part.

\subsection{Mathverse-solid}
\label{app:e1}
The findings in Table~\ref{tab:model_performance_trimmedE2} underscore the inherent challenges of Solid Geometry. GPT-5 emerges as the top-performing model with an overall accuracy of 62.9\%, followed by Gemini 2.5 pro at 59.7\% and Claude-3.7-Sonnet at 54.7\%. Notably, all other models score below 50\%, highlighting the significant hurdles that solid geometry reasoning presents even for cutting-edge large language models.

Notable performance disparities persist among state-of-the-art models across fine-grained tasks requiring robust reasoning. GPT-5 leads in SMR with a 67.1\% score yet falls short in SSI, managing only 52.9\%. Open-source models, exemplified by Deepseek-V3 67B (39.1\% in CSS and 34.7\% in SSI), consistently underperform relative to closed-source alternatives. Critically, all these models remain far behind human performance, underscoring the current challenges in replicating human-level solid geometry reasoning.
\begin{table}[ht]
\centering
\caption{MLLMs and Hilbert-Geo (Gemini-2.5-pro 45 samples) performances on Mathverse-solid; Four fine grains: Composite Solid Structures (CSS), Spatial Metric Relations (SMR), Solid Shape Identification (SSI),  Measurement of Solid Geometric Forms (MSGF); The underline represents the best performance of MLLMs}
\resizebox{\linewidth}{!}{%
\begin{tabular}{l|ccccc}

\hline
\makecell{Model} & Overall.Avg & CSS & SMR & SSI & MSGF \\ \hline
\multicolumn{6}{c}{ Closed-source MLLMs} \\ 
\hline
\makecell{GPT-4o} & 42.4 & 45.1 & 36.6 & 44.2 & 46.7\\
\makecell{GPT-5} & \underline{62.9} & \underline{64.9} & \underline{67.1} & 52.9 & 56.1 \\ 
\makecell{Claude 3.7 Sonnet} & 54.7 & 57.2 & 55.3 & 50.6 & 52.4 \\
\makecell{Gemini 2.5 Flash} & 47.5 & 50.1 & 47.2 & 44.5 & 45.9 \\
\makecell{Gemini 2.5 Pro} & 59.7 & 60.6 & 61.9 & \underline{54.1} & \underline{60.1} \\ \hline
\multicolumn{6}{c}{Open-source MLLMs} \\ \hline
\makecell{Llama 3.3 70B} & 36.0 & 38.2 & 36.5 & 33.8 & 33.1 \\
\makecell{Qwen2.5-VL-Instruct-32B} & 31.0 & 33.9 & 29.1 & 29.8 & 30.4 \\
\makecell{Qwen2.5-VL-Instruct-7B} & 22.9 & 24.7 & 20.4 & 23.5 & 24.2 \\
\makecell{Deepseek-V3 67B} & 35.4 & 39.1 & 33.7 & 34.7 & 32.2 \\ \hline
\multicolumn{6}{c}{Human Performances} \\ \hline
\makecell{Human} & 92.1 & 93.0 & 89.4 & 97.8 & 88.5 \\ \hline
\multicolumn{6}{c}{Hilbert-Geo} \\ \hline
\makecell{Hilbert-Geo (Ground-Truth) (Ours)} & 86.3 & 87.5 & 87.1 & 86.9 & 81.1 \\
\makecell{Hilbert-Geo (Gemini-2.5-Pro) (Ours)} & 84.1 & 85.9 & 84.9 & 80.2 & 78.3 \\
\hline
\end{tabular}%
}
\label{tab:model_performance_trimmedE2}
\end{table}
Hilbert-Geo, by contrast, delivers a distinct performance profile. Evaluated on Hilbert-Geo with 45 samples, it achieves top-tier results across all key dimensions: 85.9\% in CSS, 84.9\% in SMR, 80.2\% in SSI, 78.3\% in MSGF, and an overall accuracy of 84.1\%.

As shown in table~\ref{tab:model_performance_solidfgeo_avg_frontE1} and figure \ref{fig:tend11}. Based on Hilbert-Geo, the reasoning success rates of Gemini 2.5 Pro and GPT-5 rise sharply as the number of samples increases, approaching 90\% when the sample size reaches 45.
In contrast, open-source models such as Llama 3.2 70B exhibit a modest upward trend in performance, which also remains significantly lower than that of closed-source counterparts overall.
\begin{table}[ht]
\centering
\caption{Performance of Different Models on Hilbert-Geo (45 Samples): Accuracy, Average Time per Solved Problem, and Reasoning Steps; The underline represents the best performance of Hilbert-Geo}
\resizebox{\linewidth}{!}{%
\begin{tabular}{l|ccccc}
\hline
Model & Overall.Avg & CSS & SMR & SSI & MSGF \\ \hline
\multicolumn{6}{c}{Closed-source MLLMs} \\ \hline
Hilbert-Geo (GPT-4o) & 62.1 & 65.1 & 59.6 & 64.2 & 59.1 \\
Hilbert-Geo (GPT-5) & 80.5 & 84.2 & 77.1 & \underline{82.9} & 77.5 \\
Hilbert-Geo (Claude 3.7) & 73.2 & 76.2 & 69.8 & 74.6 & 72.4 \\
Hilbert-Geo (Gemini 2.5 Flash) & 69.1 & 70.1 & 67.5 & 71.5 & 67.9\\
Hilbert-Geo (Gemini 2.5 Pro) & \underline{84.1} & \underline{85.9} & \underline{84.9} & 80.2 & \underline{78.3} \\ \hline
\multicolumn{6}{c}{Open-source MLLMs} \\ \hline
Hilbert-Geo (Llama 3.3 70B) & 52.0 & 54.6 & 46.5 & 53.8 & 56.1\\
Hilbert-Geo (Qwen2.5-Vl-32B) & 38.7 & 43.9 & 36.1 & 35.8 & 36.4 \\
Hilbert-Geo (Qwen2.5-Vl-7B) & 29.4 & 30.4 & 26.9 & 33.7 & 28.2 \\
Hilbert-Geo (Deepseek-) & 46.6 & 49.1 & 44.7 & 44.7 & 47.2  \\ \hline
\multicolumn{6}{c}{Formal Language Ground Truth} \\ \hline
Hilbert-Geo (CDL) & 86.1 & 87.5 & 87.1 & 86.9 & 81.1 \\ \hline
\end{tabular}%
}
\label{tab:model_performance_solidfgeo_avg_frontE1}
\end{table}

\begin{figure}[ht]
    \centering
    \includegraphics[width=0.93\linewidth]{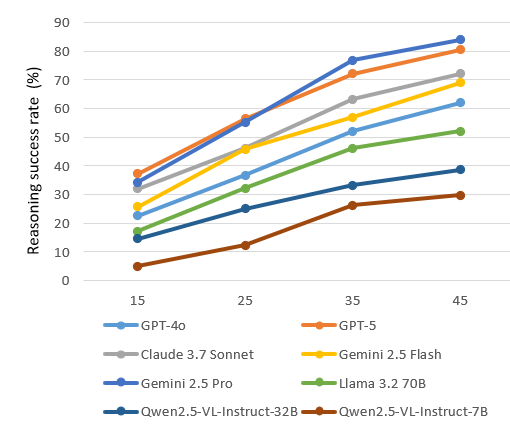}
    \caption{Reasoning performance of MLLMs under different numbers of samples based on Hilbert-Geo}
    \label{fig:tend11}
\end{figure}

\clearpage
\subsection{PlaneFGeo3k}
\label{app:e2}
\begin{table}[htbp]
  \centering
  \caption{MLLMs and Hilbert-Geo (Gemini-2.5-pro 45 samples) performances on PlaneFGeo3k}
  \label{tab:solid_geo_overall_accuracy}
  % 自适应页面宽度，避免列溢出（如需固定宽度可调整为 \resizebox{12cm}{!}）
  \resizebox{\linewidth}{!}{
    \begin{tabular}{lc}
      \toprule  % 美观粗线条（booktabs宏包），无宏包则替换为 \hline
      Model & Accuracy (\%) \\
      \midrule  % 中间细线条
      GPT-4o & 47.2 \\
      GPT-5 & 68.1 \\
      Claude 3.7 Sonnet & 60.0 \\
      Gemini 2.5 Flash & 55.6 \\
      Gemini 2.5 Pro & 72.3 \\
      Llama 3.2 70B & 33.3 \\
      Qwen2.5-VL-Instruct-32B & 24.8 \\
      Qwen2.5-VL-Instruct-7B & 17.5 \\
      Deepseek-V3 67B & 30.4 \\
      human & 87.9 \\
      Ground Truth & 84.1 \\
      Hilbert-Geo & 80.2 \\
      \bottomrule  % 底部粗线条
    \end{tabular}
  }
\end{table}
As shown in table~\ref{tab:solid_geo_overall_accuracy}, GPT-5 emerges as the top performer with an overall accuracy of 72.3\%, followed by Gemini 2.5 Pro at 68.1\% and Claude 3.7 Sonnet at 60.0\%. 
Notable performance disparities persist among leading models. Open-source models like Llama 3.2 70B consistently underperform compared to closed-source ones. Critically, all these models remain far behind human performance, highlighting the current difficulties in replicating human-level plane geometry reasoning.
Hilbert-Geo, by contrast, delivers a distinct performance profile. Evaluated on Plane-Geo with 45 samples, it achieves excellent results across all key areas, with an overall accuracy of 80.2\%.

%% file: main.bib
@String(CVPR= {IEEE Conf. Comput. Vis. Pattern Recog.})

@String(ICCV= {Int. Conf. Comput. Vis.})

@String(AAAI = {AAAI})

@String(CVPR  = {CVPR})

@String(ICCV  = {ICCV})

@InProceedings{vidhalluc2025,
    author    = {Li, Chaoyu and Im, Eun Woo and Fazli, Pooyan and others},
    title     = {VidHalluc: Evaluating Temporal Hallucinations in Multimodal Large Language Models for Video Understanding},
    booktitle = {Proceedings of the IEEE/CVF Conference on Computer Vision and Pattern Recognition (CVPR)},
    month     = {June},
    year      = {2025},
    pages = {13723--13733}
}

@article{zhao2023survey,
  author  = {Zhao, Wayne Xin and Zhou, Kun and Li, Junyi and others},
  title   = {A Survey of Large Language Models},
  journal = {arXiv preprint arXiv:2303.18223},
  year    = {2023}
}

@article{hadi2023survey,
  author  = {Hadi, Muhammad Usman and Al Tashi, Qasem and Shah, Abbas and others},
  title   = {Large Language Models: A Comprehensive Survey of its Applications, Challenges, Limitations, and Future Prospects},
  journal = {TechRxiv},
  year    = {2024},
  doi     = {10.36227/techrxiv.23589741.v6},
  url     = {https://doi.org/10.36227/techrxiv.23589741.v6},
  note    = {Preprint, version 6}
}

@article{wu2024survey,
  author  = {Wu, Likang and Zheng, Zhi and Qiu, Zhaopeng and others},
  title   = {A Survey on Large Language Models for Recommendation},
  journal = {World Wide Web},
  volume  = {27},
  number  = {5},
  pages   = {60},
  year    = {2024}
}

@article{jiang2024survey,
  author  = {Jiang, Juyong and Wang, Fan and Shen, Jiasi and others},
  title   = {A Survey on Large Language Models for Code Generation},
  journal = {ACM Transactions on Software Engineering and Methodology},
  volume  = {35},
  number  = {2},
  pages   = {1--72},
  year    = {2026},
  url     = {https://dl.acm.org/doi/10.1145/3747588}
}

@article{naveed2023comprehensive,
  author  = {Naveed, Humza and Khan, Asad Ullah and Qiu, Shi and others},
  title   = {A Comprehensive Overview of Large Language Models},
  journal = {ACM Transactions on Intelligent Systems and Technology},
  volume  = {16},
  number  = {5},
  pages   = {1--72},
  year    = {2025},
  url     = {https://dl.acm.org/doi/10.1145/3744746}
}

@misc{openai2024hello,
  author       = {OpenAI},
  title        = {Introducing GPT-5},
  year         = {2025},
  howpublished = {\url{https://openai.com/index/introducing-gpt-5/}},
  note         = {OpenAI announcement}
}

@inproceedings{lu2023mathvista,
  author    = {Lu, Pan and Bansal, Hritik and Xia, Tony and others},
  title     = {{MathVista}: Evaluating Mathematical Reasoning of Foundation Models in Visual Contexts},
  booktitle = {The Twelfth International Conference on Learning Representations},
  year      = {2024},
  note      = {Oral presentation},
  url       = {https://openreview.net/forum?id=KUNzEQMWU7}
}

@inproceedings{chen2021geoqa,
  author    = {Chen, Jiaqi and Tang, Jianheng and Qin, Jinghui and Liang, Xiaodan and Liu, Lingbo and Xing, Eric P. and Lin, Liang},
  title     = {{GeoQA}: A Geometric Question Answering Benchmark Towards Multimodal Numerical Reasoning},
  booktitle = {Findings of the Association for Computational Linguistics: ACL-IJCNLP 2021},
  pages     = {513--523},
  year      = {2021}
}

@inproceedings{chen2022unigeo,
  author    = {Chen, Jiaqi and Li, Tong and Qin, Jinghui and others},
  title     = {{UniGeo}: Unifying Geometry Logical Reasoning via Reformulating Mathematical Expression},
  booktitle = {Proceedings of the 2022 Conference on Empirical Methods in Natural Language Processing},
  pages     = {3313--3323},
  year      = {2022}
}

@article{zhang2023formalgeo,
  author  = {Zhang, Xiaokai and Zhu, Na and He, Yiming and others},
  title   = {{FormalGeo}: An Extensible Formalized Framework for Olympiad Geometric Problem Solving},
  journal = {arXiv preprint arXiv:2310.18021},
  year    = {2023},
  url      = {https://arxiv.org/abs/2310.18021}
}

@book{badescu2004projective,
  author    = {B{\u a}descu, Lucian},
  title     = {Projective Geometry and Formal Geometry},
  publisher = {Birkh{\"a}user Basel},
  year      = {2004}
}

@inproceedings{lu2021intergps,
  author    = {Lu, Pan and Gong, Ran and Jiang, Shibiao and others},
  title     = {{Inter-GPS}: Interpretable Geometry Problem Solving with Formal Language and Symbolic Reasoning},
  booktitle = {Proceedings of the 59th Annual Meeting of the Association for Computational Linguistics and the 11th International Joint Conference on Natural Language Processing (Volume 1: Long Papers)},
  pages     = {6774--6786},
  year      = {2021},
  address   = {Online},
  publisher = {Association for Computational Linguistics},
  doi       = {10.18653/v1/2021.acl-long.528},
  url       = {https://aclanthology.org/2021.acl-long.528/}
}

@inproceedings{wang2025mv,
  author    = {Wang, Ke and Pan, Junting and Shi, Weikang and others},
  title     = {Measuring Multimodal Mathematical Reasoning with MATH-Vision Dataset},
  booktitle = {NeurIPS 2024 Datasets and Benchmarks Track},
  year      = {2024},
  doi       = {10.52202/079017-3014},
  url       = {https://openreview.net/forum?id=QWTCcxMpPA}
}

@article{pittalis2010types,
  author    = {Pittalis, M. and Christou, C.},
  title     = {Types of reasoning in 3D geometry thinking and their relation with spatial ability},
  journal   = {Educational Studies in Mathematics},
  volume    = {75},
  number    = {2},
  pages     = {191--212},
  year      = {2010}
}

@inproceedings{murphy2024autoformalizing,
  author    = {Murphy, Logan and Yang, Kaiyu and Sun, Jialiang and others},
  title     = {{Autoformalizing} Euclidean Geometry},
  booktitle = {Proceedings of the 41st International Conference on Machine Learning},
  series     = {Proceedings of Machine Learning Research},
  volume     = {235},
  pages      = {36847--36893},
  year       = {2024},
  publisher  = {PMLR},
  url        = {https://proceedings.mlr.press/v235/murphy24a.html}
}

@inproceedings{zhang2024mathverse,
  author    = {Zhang, Renrui and Jiang, Dongzhi and Zhang, Yichi and others},
  title     = {{MathVerse}: Does Your Multi-modal LLM Truly See the Diagrams in Visual Math Problems?},
  booktitle = {European Conference on Computer Vision},
  pages     = {169--186},
  year      = {2024},
  publisher = {Springer},
  doi       = {10.1007/978-3-031-73242-3_10},
  url       = {https://doi.org/10.1007/978-3-031-73242-3_10}
}

@inproceedings{wang2025solidgeo,
  author    = {Wang, Peijie and Yang, Chao and Li, Zhong-Zhi and others},
  title     = {{SolidGeo}: Measuring Multimodal Spatial Math Reasoning in Solid Geometry},
  booktitle = {NeurIPS 2025 Datasets and Benchmarks Track},
  year      = {2025},
  note      = {Poster},
  url       = {https://openreview.net/forum?id=74f76e48ea09974b736b0bb829508b7c4a695fd6}
}

@article{gemini2023,
  author  = {Gemini Team},
  title   = {Gemini: A Family of Highly Capable Multimodal Models},
  journal = {arXiv preprint arXiv:2312.11805},
  year    = {2023}
}

@misc{openai2024gpt4o,
  author       = {OpenAI},
  title        = {Hello GPT-4o},
  year         = {2024},
  howpublished = {\url{https://openai.com/index/hello-gpt-4o/}},
  note         = {OpenAI announcement}
}

@article{bai2025qwen,
  author  = {Bai, Shuai and Chen, Keqin and Liu, Xuejing and others},
  title   = {{Qwen2.5-VL} Technical Report},
  journal = {arXiv preprint arXiv:2502.13923},
  year    = {2025}
}

@misc{openai2025gpt5,
  author       = {OpenAI},
  title        = {{GPT-5} System Card},
  year         = {2025},
  howpublished = {\url{https://openai.com/index/gpt-5-system-card/}},
  note         = {OpenAI system card}
}

@misc{meta2025llama3.3,
  author       = {Meta},
  title        = {{Llama 3.3} Model Cards and Prompt Formats},
  year         = {2024},
  note         = {Official Meta documentation for Llama 3.3, release date: December 6, 2024}
}

@article{deepseek2024v3,
  author  = {DeepSeek-AI},
  title   = {{DeepSeek-V3}  Technical Report},
  journal = {arXiv preprint arXiv:2412.19437},
  year    = {2024},
  url     = {https://arxiv.org/abs/2412.19437}
}

@book{hibi2003groebner,
  editor    = {Hibi, Takayuki},
  title     = {Gr\"obner Bases: Statistics and Software Systems},
  publisher = {Springer Tokyo},
  year      = {2014},
  doi       = {10.1007/978-4-431-54574-3},
  url       = {https://doi.org/10.1007/978-4-431-54574-3},
  note      = {Copyright 2013}
}

@article{cohen1960coefficient,
  title={A coefficient of agreement for nominal scales},
  author={Cohen, Jacob},
  journal={Educational and Psychological Measurement},
  volume={20},
  number={1},
  pages={37--46},
  year={1960},
  publisher={SAGE Publications Sage CA: Los Angeles, CA},
  doi={10.1177/001316446002000104}
}

@article{wu2025nesygeo,
  author    = {Wu, Weiming and Ye, Jiachen and Wang, Zihao and Zhou, Ziyi and Li, Yifan and Guo, Luzhen},
  title     = {NeSyGeo: A Neuro-Symbolic Framework for Multimodal Geometric Reasoning Data Generation},
  journal   = {arXiv preprint arXiv:2505.17121},
  year      = {2025}
}

@incollection{hasan2015formal,
  author    = {Hasan, Osman and Tahar, Sofiene},
  title     = {Formal Verification Methods},
  booktitle = {Encyclopedia of Information Science and Technology, Third Edition},
  publisher = {IGI Global Scientific Publishing},
  year      = {2015},
  pages     = {7162--7170}
}

@article{jian2023solving,
  author  = {Jian, Pengpeng and Guo, Fucheng and Wang, Yanli and others},
  title   = {Solving Geometry Problems via Feature Learning and Contrastive Learning of Multimodal Data},
  journal = {Computer Modeling in Engineering \& Sciences},
  volume  = {136},
  number  = {2},
  pages   = {1707--1728},
  year    = {2023},
  doi     = {10.32604/cmes.2023.023243},
  url     = {https://doi.org/10.32604/cmes.2023.023243}
}

@article{battaglia2018relational,
  author  = {Battaglia, Peter W. and Hamrick, Jessica B. and Bapst, Victor and others},
  title   = {Relational Inductive Biases, Deep Learning, and Graph Networks},
  journal = {arXiv preprint arXiv:1806.01261},
  year    = {2018}
}

@inproceedings{li2024survey,
  author    = {Li, Zhaoyu and Sun, Jialiang and Murphy, Logan and others},
  title     = {A Survey on Deep Learning for Theorem Proving},
  booktitle = {Proceedings of the First Conference on Language Modeling},
  year      = {2024},
  url       = {https://openreview.net/forum?id=zlw6AHwukB}
}

@inproceedings{chen2024overthinking,
  author    = {Chen, Xingyu and Xu, Jiahao and Liang, Tian and others},
  title     = {Do {NOT} Think That Much for 2+3=? On the Overthinking of Long Reasoning Models},
  booktitle = {Proceedings of the 42nd International Conference on Machine Learning},
  series    = {Proceedings of Machine Learning Research},
  volume    = {267},
  pages     = {9487--9499},
  year      = {2025},
  publisher = {PMLR},
  url       = {https://proceedings.mlr.press/v267/chen25bx.html}
}

@article{wang2025underthinking,
  author    = {Wang, Yue and Liu, Qiuzhi and Xu, Jiahao and Liang, Tian and Chen, Xingyu and He, Zhiwei and Song, Linfeng and Yu, Dian and Li, Juntao and Zhang, Zhuosheng and others},
  title     = {Thoughts Are All Over the Place: On the Underthinking of O1-like LLMs},
  journal   = {arXiv preprint arXiv:2501.18585},
  year      = {2025}
}

@inproceedings{ning2023symbolic,
  title={A Symbolic Characters Aware Model for Solving Geometry Problems},
  author={Ning, Maizhen and Wang, Qiu-Feng and Huang, Kaizhu and Huang, Xiaowei},
  booktitle={Proceedings of the 31st ACM International Conference on Multimedia (MM '23)},
  pages={7767--7775},
  year={2023},
  publisher={ACM},
  address={New York, NY, USA},
  doi={10.1145/3581783.3612570},
  url={https://dl.acm.org/doi/abs/10.1145/3581783.3612570},
  isbn={9798400701085}
}

@inproceedings{ning2025gns,
  title={{GNS}: Solving Plane Geometry Problems by Neural-Symbolic Reasoning with Multi-Modal LLMs},
  author={Ning, Maizhen and Zhou, Zihao and Wang, Qiufeng and Huang, Xiaowei and Huang, Kaizhu},
  booktitle={Proceedings of the AAAI Conference on Artificial Intelligence},
  volume={39},
  number={23},
  pages={24957--24965},
  year={2025},
  doi={10.1609/aaai.v39i23.34679},
  url={https://ojs.aaai.org/index.php/AAAI/article/view/34679}
}

@misc{google2025gemini25,
  author       = {Google},
  title        = {Gemini 2.5: Updates to Our Family of Thinking Models},
  year         = {2025},
  howpublished = {\url{https://developers.googleblog.com/en/gemini-2-5-thinking-model-updates/}},
  note         = {Introduces Gemini 2.5 Pro and Gemini 2.5 Flash updates}
}

@misc{anthropic2025claude37,
  author       = {Anthropic},
  title        = {{Claude} 3.7 Sonnet System Card},
  year         = {2025},
  howpublished = {\url{https://www.anthropic.com/claude-3-7-sonnet-system-card}},
  note         = {System card for Claude 3.7 Sonnet}
}

@article{arnon1984cad,
  author  = {Arnon, Dennis S. and Collins, George E. and McCallum, Scott},
  title   = {Cylindrical Algebraic Decomposition I: The Basic Algorithm},
  journal = {SIAM Journal on Computing},
  volume  = {13},
  number  = {4},
  pages   = {865--877},
  year    = {1984},
  doi     = {10.1137/0213054}
}

@inproceedings{xia2025geox,
  author    = {Xia, Renqiu and Li, Mingsheng and Ye, Hancheng and others},
  title     = {{GeoX}: Geometric Problem Solving Through Unified Formalized Vision-Language Pre-training},
  booktitle = {The Thirteenth International Conference on Learning Representations},
  year      = {2025},
  url       = {https://openreview.net/forum?id=6RiBl5sCDF}
}

@inproceedings{zhao2025pigps,
  author    = {Zhao, Junbo and Zhang, Ting and Sun, Jiayu and Tian, Mi and Huang, Hua},
  title     = {{Pi-GPS}: Enhancing Geometry Problem Solving by Unleashing the Power of Diagrammatic Information},
  booktitle = {Proceedings of the IEEE/CVF International Conference on Computer Vision (ICCV)},
  pages     = {1526--1536},
  year      = {2025},
  month     = {October},
  url       = {https://openaccess.thecvf.com/content/ICCV2025/papers/Zhao_Pi-GPS_Enhancing_Geometry_Problem_Solving_by_Unleashing_the_Power_of_ICCV_2025_paper.pdf}
}
